\documentclass[12pt]{article}
\usepackage[colorlinks]{hyperref}
\usepackage{float}
\usepackage{mwe}
\usepackage{color}
\usepackage{graphicx,amsmath,amssymb,amsfonts,bm,epsfig,epsf,url,dsfont}
\usepackage{fullpage}
\usepackage[small,bf]{caption}
\usepackage{subcaption}
\usepackage[top=1in,bottom=1in,left=1in,right=1in]{geometry}
\usepackage[sort&compress]{natbib} 
\usepackage{bbm}
\usepackage{booktabs}
\usepackage{cases}
\usepackage{dsfont}
\usepackage{multirow}

\newtheorem{theorem}{Theorem}

\newtheorem{lemma}{Lemma}

\newenvironment{proof}{\medskip\noindent{\bf Proof.}}{\hfill$\Box$\vspace*{1mm}\medskip}

\newcommand{\mI}{\mathrm{I}}

\renewcommand{\phi}{\varphi}

\renewcommand{\P}{\mathbb{P}}
\newcommand{\E}{\mathbb{E}}

\newcommand{\R}{\mathbb{R}}

\newcommand{\cN}{\mathcal{N}}

\def\ds1{\mathds{1}}
\renewcommand{\epsilon}{\varepsilon}

\renewcommand{\tilde}{\widetilde}

\newlength{\minipagewidth}
\newcommand{\beq}{\begin{equation}}
\newcommand{\eeq}{\end{equation}}

\newcommand{\beqa}{\begin{eqnarray}}
\newcommand{\eeqa}{\end{eqnarray}}

\newcommand{\beqan}{\begin{eqnarray*}}
\newcommand{\eeqan}{\end{eqnarray*}}

\def\ba#1\ea{\begin{align*}#1\end{align*}} %\ba = \begin{algin*}, \ea = \end{align*}
\def\banum#1\eanum{\begin{align}#1\end{align}} %\banum = \begin{algin}, \eanum

\newcommand{\mS}{\mathbb{S}}

\begin{document}
\title{A single gradient step finds adversarial examples on random two-layers neural networks}
\author{S\'ebastien Bubeck \\
Microsoft Research
\and Yeshwanth Cherapanamjeri \\
UC Berkeley
\and  Gauthier Gidel\thanks{Canada CIFAR AI Chair} \\
Mila, Université de Montréal \\
\and R\'emi Tachet des Combes \\
Microsoft Research}
\date{}

\maketitle

\begin{abstract}
Daniely and Schacham recently showed that gradient descent finds adversarial examples on random undercomplete two-layers ReLU neural networks. The term ``undercomplete'' refers to the fact that their proof only holds when the number of neurons is a vanishing fraction of the ambient dimension.
We extend their result to the overcomplete case, where the number of neurons is larger than the dimension (yet also subexponential in the dimension). In fact we prove that a single step of gradient descent suffices. We also show this result for any subexponential width random neural network with smooth activation function.
\end{abstract}

\section{Introduction} \label{sec:intro}
We study the following random two-layers neural network model: let $f : \R^d \rightarrow \R$ be a random function defined by
\begin{equation} \label{eq:nnform}
f(x) = \frac1{\sqrt{k}} \sum_{\ell=1}^k a_\ell \psi(w_\ell \cdot x) \,,
\end{equation}
where $\psi : \R \rightarrow \R$ is a fixed non-linearity, the weight vectors $w_{\ell} \in \R^d$ are i.i.d.\ from a Gaussian distribution $\cN\left(0, \frac{1}{d} \mI_d\right)$ (so that they are roughly unit norm vectors), and the coefficients $a_{\ell} \in \R$ are independent from the weight vectors and i.i.d.\ uniformly distributed in $\left\{ - 1 , +1 \right\}$. With this parametrization, the central limit theorem says that, for $x \in \sqrt{d} \cdot \mS^{d-1}$ (so that $w_{\ell} \cdot x \sim \cN(0,1)$) and large width $k$, the distribution of $f(x)$ is approximately a centered Gaussian with variance $\E_{X \sim \cN(0,1)}[\psi(X)^2]$.
\newline

Our goal is to study the concept of {\em adversarial examples} in this random model. We say that $\delta \in \R^d$ is an {\em adversarial perturbation} at $x \in \R^d$ if $\|\delta\| \ll \|x\|$ and $\mathrm{sign}(f(x)) \neq \mathrm{sign}(f(x+\delta))$, and in this case we call $x+\delta$ an {\em adversarial example}. Our main result is that, while $|f(x)| \approx 1$ with high probability, a {\em single} gradient step on $f$ (i.e., a perturbation of the form $\delta = \eta \nabla f(x)$ for some $\eta \in \R$) suffices to find such adversarial examples, with roughly $\|\delta\| \simeq \frac{\|x\|}{\sqrt{d}}= 1$. We prove this statement for a network with smooth non-linearity and subexponential width (e.g., $k \ll \exp(o(d))$), as well as for the Rectified Linear Unit (ReLU) $\psi(t) = \max(0,t)$ in the overcomplete and subexponential regime (e.g. $d \ll k \ll \exp(d^{c})$ for some constant $c>0$).

\begin{theorem} \label{thm:smooth}
Let $\gamma \in (0,1)$ and $\psi$ be non-constant, Lipschitz and with Lipschtiz derivative. There exists constants $C_1, C_2, C_3, C_4$ depending on $\psi$ such that the following holds true. Assume $k \geq C_1 \log^3(1/\gamma)$ and $d \geq C_2 \log(k/\gamma) \log(1/\gamma)$, and let $\eta \in \R$ such that $|\eta| = C_3 \sqrt{\log(1/\gamma)}$ and $\mathrm{sign}(\eta) = - \mathrm{sign}(f(x))$. Then with probability at least $1-\gamma$ one has:
\[
\mathrm{sign}(f(x)) \neq \mathrm{sign}(f(x+\eta \nabla f(x))) \,.
\]
Moreover we have $\|\eta \nabla f(x)\| \leq C_4 \sqrt{\log(1/\gamma)}$.
\end{theorem}
Note that our proof of Theorem \ref{thm:smooth} in Section \ref{sec:smooth} easily gives explicit values for $C_1, C_2, C_3, C_4$. Also note that the subexponential width condition in the above Theorem is of the form $k \ll \exp(o(d))$.

\begin{theorem} \label{thm:relu}
Let $\gamma \in (0,1)$ and $\psi(t) = \max(0,t)$. There exists constants $C_1, C_2, C_3, C_4$ such that the following holds true. Assume $k \geq C_1 d \log^2(d)$ and $\frac{d}{\log(d)} \geq C_2 \log^4(k) \log(1/\gamma)$, and let $\eta \in \R$ such that $|\eta| = C_3 \sqrt{\log(1/\gamma)}$ and $\mathrm{sign}(\eta) = - \mathrm{sign}(f(x))$. Then with probability at least $1-\gamma$ one has:
\[
\mathrm{sign}(f(x)) \neq \mathrm{sign}(f(x+\eta \nabla f(x))) \,.
\]
Moreover we have $\|\eta \nabla f(x)\| \leq C_4 \sqrt{\log(1/\gamma)}$.
\end{theorem}
Note that the subexponential condition on the width in the above Theorem is of the form $k \ll \exp(d^{0.24})$. In fact by modifying a bit the proof we can get a condition of the form $k \ll \exp(d^{\rho})$ for any $\rho<1/2$, but for the sake of clarity we only prove the weaker version stated above.

\subsection{Related works}
The existence of adversarial examples in neural network architectures was first evidenced in the seminal paper of \cite{G14}, where the authors found adversarial examples by using the L-BFGS optimization procedure. Shortly after this work, it was hypothesized in \cite{Goodfellow15} that the existence of adversarial examples stems from an excessive ``linearity" of neural network models. This hypothesis was experimentally confirmed by showing that a {\em single step} of gradient descent suffices to find adversarial perturbations (the so-called {\em fast gradient sign method} -FGSM-). Our theorems can be thought of as a theoretical confirmation of the hypothesis in \cite{Goodfellow15}. In fact, as we explain in Section \ref{sec:landscape} below, our proofs proceed exactly by showing that ``most" two-layers neural networks behave ``mostly" linearly over ``vast" regions of input space.
\newline

We note that not {\em all} networks are susceptible to one-step gradient attacks to find adversarial examples. Indeed, in \cite{Goodfellow15}, it was shown that adversarial training can be used to build networks that are somewhat robust to one-step gradient attacks. Interestingly in \cite{Madry18} it was then shown that such models remain in fact susceptible to {\em multi-steps} gradient attacks, and empirically they demonstrated that better robustness can be achieved with adversarial training using multi-steps gradient attacks. Understanding this phenomenon theoretically remains a challenge, see for example \cite{allen2020feature} for one proposed approach, and \cite{8953595,NEURIPS2019_0defd533} for discussion/algorithmic consequences of the relation with the phenomenon of {\em gradient obfuscation} (see \cite{papernot2017practical, athalye2018obfuscated}).
\newline

Our work is a direct follow-up of \cite{DanielySchacham} (which itself is a follow-up on \cite{ShamirFather}). Daniely and Schacham prove that multi-steps gradient descent finds adversarial examples for ReLU random networks of the form \eqref{eq:nnform}, as long as the number of neurons is much smaller than the dimension (i.e., $k = o(d)$). They explicitly conjecture that this condition is not necessary, and indeed we exponentially improve their condition to requiring $k = \exp(o(d))$) in Theorem \ref{thm:relu} (see below for a discussion of the case where $k$ is exponential in the dimension). We note that there remains a small window of widths around $k \simeq d$ where the conjecture of Daniely and Schacham remains open, as we require $k \geq d \log^2(d)$ in Theorem \ref{thm:relu}. Moreover Daniely and Schacham went beyond two-layers neural networks, and they conjecture (and prove for shrinking layers) that gradient descent finds adversarial examples on random multi-layers neural networks. We give some experimental confirmation of this multi-layer conjecture in Section \ref{sec:exp}.
\newline

 The ultra-wide case $k = \exp(\Omega(d))$ remains open. This exponential size case seems of a different nature than the polynomial size we tackle here, at least for the ReLU activation function. In particular it is likely that the behavior with exponential width would be closely tied to the actual limit case $k=+\infty$, where the random model \eqref{eq:nnform} yields a Gaussian process. Namely for $k=+\infty$ one has that $f$ is a Gaussian process indexed by the sphere (say if we restrict to inputs $x \in \sqrt{d} \cdot \mS^{d-1}$), with $f(x) \sim \cN(0, \E_{X \sim \cN(0,1)}[\psi(X)])$ and $\E[f(x) f(y)] = \E_{X, Y \sim \cN(0,1) : \E[X Y] = x \cdot y} [\psi(X) \psi(Y)]$. For example if the activation function is a Hermite polynomial of degree $p$, then $f$ would be a spherical $p$-spin glass model. This polynomial case is particularly well-understood, and in fact the landscape we describe below in Section \ref{sec:landscape} was already described in this case in \cite{arous2020geometry} (see in particular Corollary 59). It would be interesting to see if the $p$-spin glass landscape literature can be extended to non-polynomial activation functions, and to a finite (but possibly exponential in $d$) $k$. A step in this latter direction was recently taken in \cite{EMS21}, where convergence rates to the Gaussian process limit where given both for polynomial activations and for the ReLU. Finally we note that for a smooth activation it might be that there is a more direct argument to remove the subexponential width condition in Theorem \ref{thm:smooth} (technically in the proof of Lemma \ref{lem:hesssmooth2} there might a better argument than using the naive upper bound on $\mathrm{Lip}(\Phi)$).

Finally, we note that, in practice, it has been found that there exists ``universal" adversarial perturbations that generalize across both inputs and neural networks, \cite{moosavi2017universal}. For the case of ReLU activation (Theorem \ref{thm:relu}), we could in fact prove our result by replacing the gradient step with a step in the direction $\sum_{\ell=1}^k a_{\ell} w_{\ell}$, which is indeed a direction {\em independent} of the input $x$, thus proving the existence of ``universal" perturbation (generalizing across inputs) for our model.

\subsection{The landscape of random two-layers neural networks} \label{sec:landscape}
For a smooth $\psi$, we have
\[
\nabla f(x) = \frac{1}{\sqrt{k}} \sum_{\ell=1}^k a_{\ell} w_{\ell} \psi'(w_{\ell} \cdot x) \,,
\]
and
\[
\nabla^2 f(x) = \frac{1}{\sqrt{k}} \sum_{\ell=1}^k a_{\ell} w_{\ell} w_{\ell}^{\top} \psi''(w_{\ell} \cdot x) \,.
\]
We already claimed in the introduction that, with high probability,
\begin{equation} \label{eq:value}
|f(x)| = O(1) \,.
\end{equation}
We alluded to the CLT for this claim, but it is also easy to guess it intuitively by noting that (since $\E[a_{\ell} a_{\ell'}] = \ds1\{\ell = \ell'\}$):
\[
\E[f(x)^2] = \E \left[\frac{1}{k} \sum_{\ell,\ell'=1}^k a_{\ell} a_{\ell'} \psi(w_{\ell} \cdot x) \psi(w_{\ell'} \cdot x) \right] = \E_{X \sim \cN(0,1)}[\psi(X)^2] \,.
\]
The formal proof of \eqref{eq:value} (and all other claims we make here) will eventually be a simple application of the classical Bernstein concentration inequality. Similarly, it is easy to see that (note for example that $\E [\|\nabla f(x)\|^2] = \E_{X \sim \cN(0,1)}[\psi'(X)^2]$), with high probability,
\begin{equation} \label{eq:grad}
\|\nabla f(x)\| = \Theta(1) \,.
\end{equation}
A slightly more difficult calculation, although classical too, is that
\begin{equation} \label{eq:opnorm}
\| \nabla^2 f(x) \|_{\mathrm{op}} = \tilde{O}\left( \frac{1}{\sqrt{d}}\right) \,.
\end{equation}
Indeed one can simply note that, for any $u \in \mS^{d-1}$, $u^{\top} \nabla^2 f(x) u = \frac{1}{\sqrt{k}} \sum_{\ell=1}^k a_{\ell} (w_{\ell} \cdot u)^2 \psi''(w_{\ell} \cdot x)$ is approximately distributed as a centered Gaussian with variance 
\[
\E_{X, Y \sim \cN(0,1) : \E[X Y] = x \cdot u} \left[ \left(\frac{X}{\sqrt{d}}\right)^4 \psi''\left(\frac{Y}{\sqrt{d}}\right)^2 \right] \,,
\]
so that with probability at least $1-\gamma$ one can expect $u^{\top} \nabla^2 f(x) u$ to be of order $\frac{\sqrt{\log(1/\gamma)}}{d}$, and thus by taking a union bound over a discretization of the sphere $\mS^{d-1}$, one expects the inequality \eqref{eq:opnorm}. In fact, interestingly, one can even hope that \eqref{eq:opnorm} holds true for an entire ball around $x$: with appropriate smoothness over $\psi$, this could be obtained by doing another union bound over a second discretization of a $d$-dimensional ball. In other words, we can expect that with high probability:
\begin{equation} \label{eq:opnorm2}
\forall x \in \R^d : \|x\| = \mathrm{poly}(d), \text{ one has } \| \nabla^2 f(x) \|_{\mathrm{op}} = \tilde{O}\left( \frac{1}{\sqrt{d}}\right) \,.
\end{equation}

The equations \eqref{eq:value}, \eqref{eq:grad}, and \eqref{eq:opnorm2} paint a rather clear geometric picture. There are essentially two scales around a fixed $x \in \sqrt{d} \cdot \mS^{d-1}$: The {\em macroscopic scale}, where one considers a perturbation $x+\delta$ with $\|\delta\| = \Omega(\sqrt{d})$, and the {\em mesoscopic scale} where $\|\delta\| = o(\sqrt{d})$ (we use this term because for the ReLU network there will also be a {\em microscopic scale}, with $\|\delta\| = o(1)$). At the macroscopic scale the landscape of $f$ might be very complicated, but our crucial observation is that the picture at the mesoscopic scale is dramatically simpler. Namely at the mesoscopic scale the function $f$ is essentially linear, since one has (thanks to \eqref{eq:grad} and \eqref{eq:opnorm2})
\begin{equation} \label{eq:gradsmooth}
\|\nabla f(x) - \nabla f(x+\delta)\| = o(\|\nabla f(x)\|), \forall \delta : \|\delta\| = o(\sqrt{d}) \,.
\end{equation}
Moreover, since the height of the function is constant (by \eqref{eq:value}) and the norm of the gradient is constant, it suffices to step at a constant distance in the direction of the gradient (or negative gradient) to change the sign of $f$. In other words, this already proves our main point: a single step of gradient descent (or ascent) suffices to find an adversarial example, and moreover the adversarial perturbation $\delta$ satisfies $\|\delta\| = O(1) = O(\|x\| / \sqrt{d})$. Formally one easily concludes from \eqref{eq:value}, \eqref{eq:grad}, and \eqref{eq:gradsmooth} by using the following simple lemma for gradient descent:

\begin{lemma} \label{lem:gradientdescent}
For any continuous and almost everywhere differentiable function $f$, and any $x\in \R^d$ and $\eta \in \R$, one has:
\[
\left| f\left(x + \frac{\eta}{\|\nabla f(x)\|^2} \nabla f(x) \right) - (f(x) + \eta) \right| \leq
\sup_{\delta \in \R^d : \|\delta\| \leq \frac{\eta}{\|\nabla f(x)\|}} |\eta| \frac{\|\nabla f(x) - \nabla f(x+\delta)\|}{\|\nabla f(x)\|} \,.
\]
\end{lemma}

\begin{proof}
Let $g(t) = f\left(x + t \frac{\eta}{\|\nabla f(x)\|^2} \nabla f(x) \right)$ so that
\begin{eqnarray*}
g'(t) & = & \frac{\eta}{\|\nabla f(x)\|^2} \nabla f(x) \cdot \nabla f\left(x + t \frac{\eta}{\|\nabla f(x)\|^2} \nabla f(x) \right) \\
& = & \eta + \eta \frac{\nabla f(x)}{\|\nabla f(x)\|} \cdot \frac{\nabla f\left(x + t \frac{\eta}{\|\nabla f(x)\|^2} \nabla f(x) \right) - \nabla f(x)}{\|\nabla f(x)\|} \,.
\end{eqnarray*}
Thus we have:
\[
|g(1) - g(0) - \eta| \leq \int_0^1 |g'(t)-\eta| dt \leq |\eta| \int_0^1  \frac{\left\|\nabla f\left(x + t \frac{\eta}{\|\nabla f(x)\|^2} \nabla f(x) \right) - \nabla f(x) \right\|}{\|\nabla f(x)\|} dt \,,
\]
which concludes the proof.
\end{proof}

\subsection{Proof strategy}
The starting point of the proof for both the smooth and ReLU case is to show \eqref{eq:value} and \eqref{eq:grad}, which we essentially do below in Section \ref{sec:valuegrad}. In the smooth case, one could then prove formally \eqref{eq:opnorm} and conclude as indicated in the last paragraph of Section \ref{sec:landscape}. Of course, \eqref{eq:opnorm} is simply ill-defined for the ReLU case, so one has to take a different route there. Instead we propose to directly prove \eqref{eq:gradsmooth}, that is we study the difference of gradients at the mesoscopic scale. Using that $\|h\| = \sup_{v \in \mS^{d-1}} v \cdot h$, we thus need to control (for some $R= o(\sqrt{d})$):
\begin{equation} \label{eq:gradtoprove}
\sup_{\delta \in \R^d : \|\delta\| \leq R} \|\nabla f(x) - \nabla f(x+\delta)\| = \sup_{\substack{v \in \mS^{d-1}, \\ \delta \in \R^d : \|\delta\| \leq R}} \frac{1}{\sqrt{k}} \sum_{\ell=1}^k a_{\ell} (w_{\ell} \cdot v) (\psi'(w_{\ell} \cdot x) - \psi'(w_{\ell} \cdot (x+\delta))) \,.
\end{equation}

We execute this strategy first for the smooth case in Section \ref{sec:smooth}. We then prove the ReLU case in Section \ref{sec:relu}, where we face an extraneous difficulty since the gradient is {\em not} Lipschitz at very small scale, which introduces a third scale (the {\em microscopic scale}) that has to be dealt with differently. Technically, this issue appears when we try to move from the discretization over $v$ and $\delta$ in \eqref{eq:gradtoprove} to the whole space (a so-called $\epsilon$-net argument).

\subsection{Scaling of value and gradient} \label{sec:valuegrad}
Here we show how to prove \eqref{eq:value} and \eqref{eq:grad} (in fact, for our purpose, we only need the one-sided inequality $\|\nabla f(x)\| = \Omega(1)$) under very mild conditions on $\psi$ which will be satisfied for both ReLU and smooth activations. We will repeatedly use Bernstein's inequality which we restate here for convenience (see e.g., Theorem 2.10 in \cite{BLM}):

\begin{theorem}[Bernstein's inequality]
Let $(X_{\ell})$ be i.i.d. centered random variables such that there exists $\sigma, c > 0$ such that for all integers $q \geq 2$,
\[
\E[ |X_{\ell}|^q ] \leq \frac{q!}{2} \sigma^2 c^{q-2} \,.
\]
Then with probability at least $1-\gamma$ one has:
\[
\sum_{\ell=1}^k X_{\ell} \leq \sqrt{2 \sigma^2 k \log(1/\gamma)} + c \log(1/\gamma) \,.
\]
\end{theorem}

We will also use repeatedly that $\E_{X \sim \cN(0,1)}[ |X|^q ] \leq (q-1)!! \leq \frac{q!}{2}$, as well as the following concentration of $\chi^2$ random variables (see e.g., (2.19) in \cite{Wainwright}): let $X_1, \hdots, X_k$ be i.i.d. standard Gaussians, then with probability at least $1-\gamma$, one has:
\begin{equation} \label{eq:chisquared}
\left|\sum_{\ell=1}^k X_{\ell}^2 - k \right| \leq 4 \sqrt{k \log(2/\gamma)} \,.
\end{equation}

We can now proceed to our various results.
\begin{lemma} \label{lem:value}
Assume that there exists $\sigma, c > 0$ such that for all integers $q \geq 2$,
\begin{equation} \label{eq:psi1}
\E_{X \sim \cN(0,1)} [ |\psi(X)|^q ] \leq \frac{q!}{2} \sigma^2 c^{q-2} \,.
\end{equation}
Then with probability at least $1-\gamma$ one has
\[
|f(x)| \leq \sqrt{2 \sigma^2 \log(1/\gamma)} + \frac{c \log(1/\gamma)}{\sqrt{k}} \,.
\]
\end{lemma}

\begin{proof}
Let $X_{\ell} = a_{\ell} \psi(w_{\ell} \cdot x)$. Then $\E[X_{\ell}] = 0$ and
\[
\E[|X_{\ell}|^q] \leq \frac{q!}{2} \sigma^2 c^{q-2}, \text{ for all integers } q \geq 2  \,.
\]
Thus Bernstein's inequality states that with probability at least $1-\gamma$ one has
\[
\sqrt{k} f(x) = \sum_{\ell=1}^k X_{\ell} \leq \sqrt{2 \sigma^2 k \log(1/\gamma)} + c \log(1/\gamma) \,.
\]
\end{proof}

\begin{lemma} \label{lem:grad1}
Let $\psi$ be differentiable almost everywhere.
Then with probability at least $1-\gamma$ for $0<\gamma < 2/e$ one has:
\[
\|\nabla f(x)\| \geq \left(1 - 5 \sqrt{\frac{\log(2/\gamma)}{d}}\right) \sqrt{\frac{1}{k} \sum_{\ell=1}^k \psi'(w_{\ell} \cdot x)^2} \,.
\]
\end{lemma}

\begin{proof}
Let $P= \mI_d - \frac{x x^{\top}}{d}$ be the projection on the orthogonal complement of the span of $x$. We have $\|\nabla f(x)\| \geq \|P \nabla f(x)\|$. Moreover $a_{\ell} P w_{\ell}$ is independent of $w_{\ell} \cdot x$, and thus conditioning on the values $(w_{\ell} \cdot x)_{\ell \in [k]}$ we obtain (using that $a_{\ell} P w_{\ell}$ is distributed as $\cN\left(0, \frac{1}{d} \mI_{d-1}\right)$):
\[
P \nabla f(x) = \frac{1}{\sqrt{k}} \sum_{\ell=1}^k a_{\ell} P w_{\ell} \psi'(w_{\ell} \cdot x) \stackrel{(d)}{=} \left(\sqrt{\frac{1}{k d} \sum_{\ell=1}^k \psi'(w_{\ell} \cdot x)^2} \right) Y \text{ where } Y \sim \cN\left(0, \mI_{d-1} \right) \,.
\]
Using \eqref{eq:chisquared} we have that with probability at least $1-\gamma$:
\[
\|Y\|^2 \geq d-1 - 4 \sqrt{d \log(2/\gamma)} \geq d - 5 \sqrt{d \log(2/\gamma)} \,.
\]
where we used that $d\geq 1 $ and $\gamma<2/e$.
The two above displays easily conclude the proof.
\end{proof}

\begin{lemma} \label{lem:grad2}
Let $\psi$ be differentiable almost everywhere, and assume that there exists $\sigma', c' > 0$ such that for all integers $q \geq 2$,
\[
\E_{X \sim \cN(0,1)} [ |\psi'(X)|^{2 q} ] \leq \frac{q!}{2} \sigma'^2 c'^{q-2} \,.
\]
Then with probability at least $1-\gamma$,
\[
\|\nabla f(x)\| \geq \left( \E_{X \sim \cN(0,1)}[|\psi'(X)|^2] - \left(\sqrt{\frac{2 \sigma'^2 \log(2/\gamma)}{k}} + \frac{c' \log(2/\gamma)}{k} \right) \right)^{1/2} \left(1 - 5 \sqrt{\frac{\log(4/\gamma)}{d}}\right) \,.
\]
\end{lemma}

\begin{proof}
Straightforward application of Bernstein's inequality yields with probability at least $1-\gamma$ one has:
\[
\frac{1}{k} \sum_{\ell=1}^k \psi'(w_{\ell} \cdot x)^2 \geq \E_{X \sim \cN(0,1)}[|\psi'(X)|^2] - \left(\sqrt{\frac{2 \sigma'^2 \log(1/\gamma)}{k}} + \frac{c' \log(1/\gamma)}{k} \right) \,.
\]
It suffices to combine this inequality with Lemma \ref{lem:grad1} and apply a direct union bound.
\end{proof}

\section{Proof of Theorem \ref{thm:smooth}} \label{sec:smooth}
In this section, we consider a $1$-Lipschitz and $L$-smooth activation function, that is for all $s, t \in \R$,
\begin{equation} \label{eq:smoothass1}
|\psi(s) - \psi(t)| \leq |s-t| \text{ and } |\psi'(s) - \psi'(t)| \leq L |s - t| \,.
\end{equation}
We also assume $\psi(0) = 0$ and denote $c_{\psi}^2 = \E_{X \sim \cN(0,1)}[(\psi'(X))^2]$ which we assume to be non-zero (that is $\psi$ is not a constant function).

\begin{lemma} \label{lem:fixedsmooth}
Under the above assumptions, one has with probability at least $1-\gamma$,
\[
|f(x)| \leq \sqrt{2 \log(1/\gamma)} \left(1 + \sqrt{\frac{\log(2/\gamma)}{k}}\right) \,,
\]
and
\[
\|\nabla f(x)\| \geq \left(c_{\psi}^2 - \sqrt{\frac{2 \log(4/\gamma)}{k}} \left(1 + \sqrt{\frac{\log(4/\gamma)}{k}} \right) \right)^{1/2}  \left(1 - 5 \sqrt{\frac{\log(8/\gamma)}{d}}\right)
\]
\end{lemma}

\begin{proof}
With the assumptions we have $|\psi(X)| \leq |X|$ and thus in Lemma \ref{lem:value} we can take $\sigma=c=1$ which yields the first claimed equation. For the second equation we use that $|\psi'(X)| \leq 1$ (since $\psi$ is $1$-Lipschitz) and thus, in Lemma \ref{lem:grad2}, we can also take $\sigma'=c'= 1$ which yields the second claimed equation.
\end{proof}

%In other words for $d \geq C \log(1/\gamma)$ and $k \geq C \max\left(1, \frac{m^4}{c_{\psi^2}} \right) \log(1/\gamma)$ we get $|f(x)| \leq C m \sqrt{\log(1/\gamma)}$ and $\|\nabla f(x)\| \geq \frac{c_{\psi}}{2}$.

Next we need to control \eqref{eq:gradtoprove} where we use crucially the smoothness of the activation function.

\begin{lemma} \label{lem:hesssmooth1}
Fix $\delta \in \R^d$ such that $\|\delta\| \leq R$ and $v \in \mS^{d-1}$. Then with probability at least $1-\gamma$ one has:
\[
\frac{1}{\sqrt{k}} \sum_{\ell=1}^k a_{\ell} (w_{\ell} \cdot v) (\psi'(w_{\ell} \cdot x) - \psi'(w_{\ell} \cdot (x+\delta))) \leq \frac{4 R L}{d} \sqrt{\log(1/\gamma)} \left(1 + \sqrt{\frac{\log(1/\gamma)}{k}}\right) \,.
\]
\end{lemma}

\begin{proof}
We apply Bernstein's inequality with $X_{\ell} = \frac{a_{\ell}}{L} (w_{\ell} \cdot v) (\psi'(w_{\ell} \cdot x) - \psi'(w_{\ell} \cdot (x+\delta))$. We have $\E[X_{\ell}]=0$ and (by smoothness of $\psi$)
\begin{eqnarray*}
\E[|X_{\ell}|^q] \leq \E[|w_{\ell} \cdot v|^q |w_{\ell} \cdot \delta|^q] & \leq & \sqrt{\E[|w_{\ell} \cdot v|^{2q}] \E[|w_{\ell} \cdot \delta|^{2 q}]} \\
& = & \frac{\|\delta\|^q}{d^q} \E_{X \sim \cN(0,1)}[|X|^{2 q}] \leq (2q-1)!! \left(\frac{R}{d}\right)^q \leq \frac{q!}{2} \left(\frac{2R}{d}\right)^q \,.
\end{eqnarray*}
Thus we can apply Bernstein with $\sigma = c = \frac{2R}{d}$ which yields the claimed bound.
\end{proof}

\begin{lemma} \label{lem:hesssmooth2}
Let $R \geq 1$. With probability at least $1-\gamma$ one has
\[
\sup_{v \in \mS^{d-1}, \delta \in \R^d : \delta \leq R} \frac{1}{\sqrt{k}} \sum_{\ell=1}^k a_{\ell} (w_{\ell} \cdot v) (\psi'(w_{\ell} \cdot x) - \psi'(w_{\ell} \cdot (x+\delta)))
\leq 20 R L \left( \sqrt{\frac{\log(k/\gamma)}{d}} + \frac{\log(1/\gamma)}{\sqrt{k}}\right) \,.
\]
\end{lemma}

\begin{proof}
Denote $\Phi(v,\delta) = \frac{1}{\sqrt{k}} \sum_{\ell=1}^k a_{\ell} (w_{\ell} \cdot v) (\psi'(w_{\ell} \cdot x) - \psi'(w_{\ell} \cdot (x+\delta)))$. In Lemma \ref{lem:hesssmooth1}, we controlled $\Phi(v,\delta)$ for a fixed $v$ and $\delta$. We now want to control it uniformly over $\Omega = \{(v,\delta) : \|v\| =1, \|\delta\| \leq R\}$. To do so, we apply an union bound over an $\epsilon$-net for $\Omega$, denote it $N_\epsilon$, whose size is then at most $(10R/\epsilon)^{2 d}$. In particular; we obtain with probability at least $1-\gamma$:
\begin{align}
& \sup_{(v,\delta) \in \Omega} \Phi(v,\delta) \leq \sup_{(v,\delta) \in N_\epsilon} \Phi(v,\delta) +  \sup_{(v,\delta), (v',\delta') \in \Omega : \|v-v'\| + \|\delta - \delta'\| \leq \epsilon} |\Phi(v,\delta) - \Phi(v',\delta')| \notag \\
& \leq \frac{4 R L}{d} \sqrt{2d \log(10/\epsilon) +\log(1/\gamma)} \left(1 + \sqrt{\frac{2d \log(10/\epsilon) + \log(1/\gamma)}{k}}\right) + \epsilon \times \mathrm{Lip}(\Phi) \,. \label{eq:netsmooth}
\end{align}
Thus, it only remains to estimate the Lipschitz constant of the mapping $\Phi$. To do so, note that for any $\delta, \delta'$,
\[
|\Phi(\delta,v) - \Phi(\delta',v)| \leq \frac{L \|\delta - \delta'\|}{\sqrt{k}} \sum_{\ell=1}^k \|w_{\ell}\|^2 \,,
\]
and similarly for any $v, v'$,
\[
|\Phi(\delta, v) - \Phi(\delta,v')| \leq \frac{R L \|v - v'\|}{\sqrt{k}} \sum_{\ell=1}^k \|w_{\ell}\|^2.
\]
Using \eqref{eq:chisquared}, we have with probability at least $1-\gamma$ that
\begin{equation} \label{eq:sumnormsquared}
\sum_{\ell=1}^k \|w_{\ell}\|^2 \leq k + 4 \sqrt{\frac{k \log(1/\gamma)}{d}} \,.
\end{equation}
Thus with see that with probability at least $1-\gamma$,
\[
 \mathrm{Lip}(\Phi) \leq R L \left(\sqrt{k} + 4 \sqrt{\frac{\log(1/\gamma)}{d}}\right) \,.
\]
Combining this with \eqref{eq:netsmooth} concludes the proof (by taking $\epsilon = 1/k$ and with straightforward algebraic manipulations).
\end{proof}

Finally we can turn to the proof of Theorem \ref{thm:smooth}:

\begin{proof}[of Theorem \ref{thm:smooth}]
We make the following claims which hold with probability at least $1-\gamma$. With the assumptions on $k$, $d$ and $\eta$, Lemma \ref{lem:fixedsmooth} shows that $|f(x)| \leq 0.1|\eta|$ and $\|\nabla f(x)\| \geq c$ for some small constant $c>0$. Moreover Lemma \ref{lem:hesssmooth2} shows that for all $\delta$ such that $\|\delta\| \leq |\eta| / \|\nabla f(x)\|$ we have $\|\nabla f(x) - \nabla f(x+\delta)\| \leq c/10$. Thus Lemma \ref{lem:gradientdescent} easily allows us to conclude (using in particular that $f(x) (f(x) + \eta) < 0$).
\end{proof}

\section{Proof of Theorem \ref{thm:relu}} \label{sec:relu}
In this section, we consider $\psi(t) = \max(0,t)$.

\begin{lemma} \label{lem:fixedrelu}
With probability at least $1-\gamma$,
\[
|f(x)| \leq \sqrt{2 \log(2/\gamma)} \left(1 + \sqrt{\frac{\log(2/\gamma)}{k}}\right) \,,
\]
and
\[
\|\nabla f(x)\| \geq \left(\frac{1}{2} - \sqrt{\frac{2 \log(4/\gamma)}{k}} \left(1 + \sqrt{\frac{\log(1/\gamma)}{k}} \right) \right)^{1/2}  \left(1 - 5 \sqrt{\frac{\log(4/\gamma)}{d}}\right) \,.
\]
\end{lemma}

\begin{proof}
In Lemma \ref{lem:value} and Lemma \ref{lem:grad2}, we can take $\sigma=c=\sigma'=c'=1$ (since $|\psi(X)|\leq|X|$ and $|\psi'(X)|\leq 1$), which concludes the proof.
\end{proof}

We now turn to the control of \eqref{eq:gradtoprove}. In the smooth case we did so via Lemma \ref{lem:hesssmooth1} and Lemma \ref{lem:hesssmooth2}, which both used crucially the smoothness of the activation function. Here, instead of smoothness, we will use that only few activations can change when you make {\em microscopic move} (i.e., between $x+\delta$ and $x+\delta'$ with $\|\delta - \delta'\|=o(1)$). The key observation is the following lemma:

%It turns out that Lemma \ref{lem:hesssmooth1} can easily be adapted to the ReLU case, see Lemma \ref{lem:hessrelu1} below. On the other hand for Lemma \ref{lem:hesssmooth2} a different phenomenon is at play for the ReLU, which is related to the behavior at {\em microscopic scale} (i.e. $x+ \delta$ with $\|\delta\| = o(1)$). We deal with this in Lemma \ref{lem:hessrelu2}. We start now with a result that will allow to replace a smoothness step from the proof of Lemma \ref{lem:hesssmooth1}:

\begin{lemma} \label{lem:actchange}
For any $\delta$ such that $\|\delta\| \leq R$,
\begin{equation} \label{eq:actchange1}
\P(\mathrm{sign}(w_{\ell} \cdot x) \neq \mathrm{sign}(w_{\ell} \cdot (x+\delta))) \leq R \sqrt{\frac{2\log(d)}{d}} + \frac{1}{d} \,.
\end{equation}
Moreover, for any $\delta$ with $\|\delta\| \leq \sqrt{d} /2$, we have
\begin{equation} \label{eq:actchange2}
\P(\exists \delta' :\|\delta - \delta'\| \leq \epsilon \text{ and } \mathrm{sign}(w_{\ell} \cdot (x+\delta)) \neq \mathrm{sign}(w_{\ell} \cdot (x+\delta'))) \leq 2 \epsilon \left(1 + 2 \sqrt{\frac{\log(2/\epsilon)}{d}} \right) \,.
\end{equation}
\end{lemma}

\begin{proof}
We have:
\[ \P(\mathrm{sign}(w_{\ell} \cdot x) \neq \mathrm{sign}(w_{\ell} \cdot (x+\delta)))
\leq \P(|w_{\ell} \cdot \delta| \geq |w_{\ell} \cdot x|)
\leq \P(|w_{\ell} \cdot \delta| \geq t) + \P(|w_{\ell} \cdot x| \leq t)  \,,
\]
where the last inequality holds for any threshold $t \in \R$. Now, note that $w_{\ell} \cdot \delta \sim \cN(0, \frac{\|\delta\|^2}{d})$ and $w_{\ell} \cdot x \sim \cN(0, 1)$. Thus picking $t = R \sqrt{\frac{2 \log(d)}{d}}$ shows that
\[
\P(\mathrm{sign}(w_{\ell} \cdot x) \neq \mathrm{sign}(w_{\ell} \cdot (x+\delta))) \leq R \sqrt{\frac{2\log(d)}{d}} + \frac{1}{d} \,,
\]
which concludes the proof of \eqref{eq:actchange1}.

For \eqref{eq:actchange2} we have:
\begin{align*}
& \P(\exists \delta' :\|\delta - \delta'\| \leq \epsilon \text{ and } \mathrm{sign}(w_{\ell} \cdot (x+\delta)) \neq \mathrm{sign}(w_{\ell} \cdot (x+\delta'))) \\
& \leq \P(\exists \delta' :\|\delta - \delta'\| \leq \epsilon \text{ and }  |w_{\ell} \cdot (\delta'-\delta)| \geq t) + \P(|w_{\ell} \cdot (x+\delta)| \leq t) \\
& \leq \P(\|w_{\ell}\| \geq t / \epsilon) + \P(|w_{\ell} \cdot (x+\delta)| \leq t) \,.
\end{align*}
where $w_{\ell} \cdot (x+\delta) \sim \cN(0, \sigma^2)$ with $\sigma^2 \geq \frac{1}{2}$ since $\|\delta\| \leq \sqrt{d}/{2}$. Thus picking $t= \epsilon \sqrt{1 + 4 \sqrt{\frac{\log(2/\epsilon)}{d}}}$ and applying \eqref{eq:chisquared} concludes the proof.
\end{proof}

We now give the equivalent of Lemma \ref{lem:hesssmooth1}:

\begin{lemma} \label{lem:hessrelu1}
Fix $\delta \in \R^d$ such that $\|\delta\| \leq R$ (with $R \geq 1$) and $v \in \mS^{d-1}$. Then with probability at least $1-\gamma$ one has:
\[
\frac{1}{\sqrt{k}} \sum_{\ell=1}^k a_{\ell} (w_{\ell} \cdot v) (\psi'(w_{\ell} \cdot x) - \psi'(w_{\ell} \cdot (x+\delta))) \leq 2 \sqrt{\frac{\log(1/\gamma)}{d}} \left(\left(2 R \sqrt{\frac{\log(d)}{d}}\right)^{1/4} + \sqrt{\frac{\log(1/\gamma)}{k}} \right) \,.
\]
\end{lemma}

\begin{proof}
We apply Bernstein's inequality with $X_{\ell} = a_{\ell} (w_{\ell} \cdot v) (\psi'(w_{\ell} \cdot x) - \psi'(w_{\ell} \cdot (x+\delta)))$. We have $\E[X_{\ell}]=0$ and (using \eqref{eq:actchange1} in Lemma \ref{lem:actchange})
\begin{eqnarray*}
\E[|X_{\ell}|^q] & = & \E[|w_{\ell} \cdot v|^q |\psi'(w_{\ell} \cdot x) - \psi'(w_{\ell} \cdot (x+\delta))|^q] \\
& \leq & \sqrt{\E[|w_{\ell} \cdot v|^{2 q}] \times \P(\mathrm{sign}(w_{\ell} \cdot x) \neq \mathrm{sign}(w_{\ell} \cdot (x+\delta)))} \\
& \leq & \sqrt{\frac{(2q)!}{2 d^q}} \times \sqrt{2 R \sqrt{\frac{\log(d)}{d}}}  \\
& \leq & \frac{q!}{2} \left(\frac{2}{\sqrt{d}}\right)^q  \times \sqrt{2 R \sqrt{\frac{\log(d)}{d}}} \,.
\end{eqnarray*}
Thus we can apply Bernstein with $\sigma = \frac{2}{\sqrt{d}} \times \left(2 R \sqrt{\frac{\log(d)}{d}}\right)^{1/4}$ and $c = \frac{2}{\sqrt{d}}$ which yields the claimed bound.
\end{proof}

Finally, we give the equivalent of Lemma \ref{lem:hesssmooth2}:

\begin{lemma} \label{lem:hessrelu2}
Let $1 \leq R \leq \sqrt{d}/2$, $\sqrt{k} \geq 52$ and $d \geq \log(1/\gamma)$. Then, with probability at least $1-\gamma$, one has
\begin{align*}
& \sup_{v \in \mS^{d-1}, \delta \in \R^d : \delta \leq R} \frac{1}{\sqrt{k}} \sum_{\ell=1}^k a_{\ell} (w_{\ell} \cdot v) (\psi'(w_{\ell} \cdot x) - \psi'(w_{\ell} \cdot (x+\delta)) \\
& \leq 20 \left( R\log^2 (Rk)\sqrt{\frac{\log d}{d}}\right)^{1/4} + 40 \sqrt{\frac{d}{k}} \log(Rk)\,.
% & 80 \log(k) \sqrt{\frac{d}{k}} + 30 \left(R \log(k)^2 \sqrt{\frac{\log(d)}{d}} \right)^{1/4} \,.
\end{align*}
\end{lemma}

\begin{proof}
Similarly to the proof of Lemma \ref{lem:hesssmooth2}, we define $\Phi(v,\delta) = \frac{1}{\sqrt{k}} \sum_{\ell=1}^k a_{\ell} (w_{\ell} \cdot v) (\psi'(w_{\ell} \cdot x) - \psi'(w_{\ell} \cdot (x+\delta)))$, and $N_\epsilon$ an $\epsilon$-net for $\Omega = \{(v,\delta), \|v\|=1, \|\delta\| \leq R\}$ (recall that $|N_\epsilon| \leq (10R/\epsilon)^{2 d}$). Using Lemma \ref{lem:hessrelu1}, we obtain with probability at least $1-\gamma$:
\begin{align}
& \sup_{(v,\delta) \in \Omega} \Phi(v,\delta) \leq \sup_{(v,\delta) \in N} \Phi(v,\delta) +  \sup_{(v,\delta), (v',\delta') \in \Omega : \|v-v'\| + \|\delta - \delta'\| \leq \epsilon} |\Phi(v,\delta) - \Phi(v',\delta')| \notag \\
& \leq 2 \sqrt{\frac{2 d \log(10R/\epsilon) + \log(1/\gamma)}{d}} \left(\left(2 R \sqrt{\frac{\log(d)}{d}}\right)^{1/4} + \sqrt{\frac{2 d \log(10R/\epsilon) + \log(1/\gamma)}{k}} \right) \notag \\
& \,\, + \sup_{(v,\delta), (v',\delta') \in \Omega : \|v-v'\| + \|\delta - \delta'\| \leq \epsilon} |\Phi(v,\delta) - \Phi(v',\delta')| \,. \label{eq:netrelu}
\end{align}
Thus, it remains again to estimate the ``Lipschitz constant" of the mapping $\Phi$ but crucially only {\em at scale $\epsilon$} (the crucial point is that we don't need to argue about infinitesimal scale, where a ReLU network is {\em not} smooth). For $v, v'$, one has
\[
|\Phi(\delta, v) - \Phi(\delta,v')| \leq \frac{\|v - v'\|}{\sqrt{k}} \sum_{\ell=1}^k \|w_{\ell}\|  \,.
\]
Using \eqref{eq:chisquared}, we see that with probability at least $1-\gamma$, one has for all $\ell \in [k]$,
\[
\|w_{\ell}\|^2 \leq 1 + 4 \sqrt{\frac{\log(k/\gamma)}{d}} \,,
\]
so that in this event we have:
\begin{equation}
|\Phi(\delta, v) - \Phi(\delta,v')| \leq \|v-v'\| \sqrt{k + 4 k \sqrt{\frac{\log(k/\gamma)}{d}}} \,. \label{eq:last0}
\end{equation}
On the other hand, for $\delta, \delta'$ we write:
\begin{align*}
    |\Phi(\delta,v) - \Phi(\delta',v)| &\leq \frac{1}{\sqrt{k}}\left|\sum_{\ell = 1}^k \ds1 \left\{\mathrm{sign} (w_{\ell} \cdot (x + \delta)) > \mathrm{sign} (w_{\ell} \cdot (x + \delta'))\right\} a_\ell w_\ell \cdot v \right| \\
    & + \frac{1}{\sqrt{k}} \left| \sum_{\ell = 1}^k \ds1 \left\{\mathrm{sign} (w_{\ell} \cdot (x + \delta)) < \mathrm{sign} (w_{\ell} \cdot (x + \delta'))\right\} a_\ell w_\ell \cdot v \right| \\
    &\leq \frac{1}{\sqrt{k}} \left\|\sum_{\ell = 1}^k \ds1 \left\{\mathrm{sign} (w_{\ell} \cdot (x + \delta)) > \mathrm{sign} (w_{\ell} \cdot (x + \delta'))\right\} a_\ell w_\ell \right\| \\
    &+ \frac{1}{\sqrt{k}} \left\|\sum_{\ell = 1}^k \ds1 \left\{\mathrm{sign} (w_{\ell} \cdot (x + \delta)) < \mathrm{sign} (w_{\ell} \cdot (x + \delta'))\right\} a_\ell w_\ell \right\| \refstepcounter{equation}\tag{\theequation} \label{eq:microscale_grad_dev}
\end{align*}
% \begin{align}
% |\Phi(\delta,v) - \Phi(\delta',v)| \notag \\
% & \leq \frac{1}{\sqrt{k}} \sum_{\ell=1}^k \|w_{\ell}\| \times \ds1\{\mathrm{sign}(w_{\ell} \cdot (x+\delta)) \neq \mathrm{sign}(w_{\ell} \cdot (x+\delta'))\} \notag \\
% & \leq \sqrt{\frac{1 + 4 \sqrt{\frac{\log(k/\gamma)}{d}}}{k}} \sum_{\ell=1}^k \ds1\{\mathrm{sign}(w_{\ell} \cdot (x+\delta)) \neq \mathrm{sign}(w_{\ell} \cdot (x+\delta'))\} \,. \label{eq:last1}
% \end{align}
% We now need to control (with exponentially large probability so that we can union bound over the net) the following quantity:
% \begin{equation}
% \sup_{\delta' : \|\delta - \delta'\| \leq \epsilon} \sum_{\ell=1}^k \ds1\{\mathrm{sign}(w_{\ell} \cdot (x+\delta)) \neq \mathrm{sign}(w_{\ell} \cdot (x+\delta'))\} \leq \sum_{\ell=1}^k X_{\ell} 
% \end{equation}
Letting $X_{\ell} (\delta) = \ds1\{ \exists \delta' : \|\delta - \delta'\| \leq \epsilon \text{ and } \mathrm{sign}(w_{\ell} \cdot (x+\delta)) \neq \mathrm{sign}(w_{\ell} \cdot (x+\delta')) \}$, we now control with exponentially high probability $\sum_{\ell = 1}^k X_\ell (\delta)$. By \eqref{eq:actchange2} in Lemma \ref{lem:actchange}, we know that $X_{\ell} (\delta)$ is a Bernoulli of parameter at most $2 \epsilon \left(1 + 2 \sqrt{\frac{\log(2/\epsilon)}{d}}\right)$. So we have:
\[
\P\left(\sum_{\ell=1}^k X_{\ell} (\delta) \geq s\right) \leq \left(2 k \epsilon \left(1 + 2 \sqrt{\frac{\log(2/\epsilon)}{d}}\right) \right)^s \,.
\]
And thus, thanks to an union bound, we obtain:
\begin{equation}
\P\left( \exists (v,\delta) \in N_\epsilon: \sum_{\ell=1}^k X_\ell (\delta) \geq s \right) \leq \left(\frac{10R}{\epsilon}\right)^{2 d} \left(2 k \epsilon \left(1 + 2 \sqrt{\frac{\log(2/\epsilon)}{d}}\right)\right)^s. \label{eq:last3}
\end{equation}
With $s= 4d$ the latter is upper bounded by $ (26k\sqrt{R}\epsilon^{3d/8})^{4 d}$ (using the fact that $\sqrt{\epsilon}(1+2 \sqrt{\log 2/\epsilon}) \leq 4 \epsilon^{3/8}\,,\,\forall 1\geq \epsilon >0$). Taking $\epsilon= R^{-4/3}k^{-4}$ we get that this probability is less than $(26/\sqrt{k})^{8 d} \leq \gamma$ for $\sqrt{k}\geq 52$ and $d \geq \log(1/\gamma)$.

Furthermore, we have by another union bound and the concentration of Lipschitz functions of Gaussians \cite[Theorem 5.5]{BLM} ($\|\cdot\|$ is a $1$-Lipschitz function):
\begin{equation*}
    \P \left(\exists S \subset [k],\, |S| \leq 4d: \left\|\frac{1}{\sqrt{k}} \sum_{i \in S} a_i w_i\right\| \geq \sqrt{\frac{|S|}{k}} (1 + t)\right) \leq k^{4d} e^{- \frac{dt^2}{2}}
\end{equation*}
By setting $t = 2 \sqrt{\log 4k + \frac{\log 8/\gamma}{d}}$, we get that with probability at least $1 - \gamma / 8$:
\begin{equation}
    \label{eq:small_subset_sum}
    \forall S \subset [k],\, |S| \leq 4d: \left\|\frac{1}{\sqrt{k}} \sum_{i \in S} a_i w_i\right\| \leq 9 \sqrt{\frac{d}{k}}\sqrt{\log 4k + \frac{\log 8/\gamma}{d}}
\end{equation}
Finally, noting that for all $(v,\delta) \in N, \|\delta' - \delta\| \leq \epsilon$:
\begin{gather*}
    \ds1 \left\{\mathrm{sign} (w_{\ell} \cdot (x + \delta)) < \mathrm{sign} (w_{\ell} \cdot (x + \delta'))\right\} \leq X_\ell (\delta)\\ 
    \ds1 \left\{\mathrm{sign} (w_{\ell} \cdot (x + \delta)) > \mathrm{sign} (w_{\ell} \cdot (x + \delta'))\right\} \leq X_\ell (\delta),
\end{gather*}
we may combine \eqref{eq:last0}, \eqref{eq:microscale_grad_dev}, \eqref{eq:last3} and \eqref{eq:small_subset_sum} to obtain that with probability at least $1-\gamma$, we have for all $\delta, v, \delta', v'$ with $\|\delta - \delta'\| \leq \frac{1}{R^{4/3} k^4}$ and $\|v - v'\| \leq \frac{1}{R^{4/3}k^4}$,
\[
|\Phi(\delta, v) - \Phi(\delta,v')| \leq \frac{1}{R^{4/3}k^4} \sqrt{k + 4 k \sqrt{\frac{\log(4k/\gamma)}{d}}} \,.
\]
and
\[
|\Phi(\delta, v) - \Phi(\delta',v)| \leq 18 \sqrt{\frac{d}{k}} \sqrt{\log 4k + \frac{\log 8/\gamma}{d}}\,.
\]
Combining this with \eqref{eq:netrelu} we obtain with probability at least $1-\gamma$:
\begin{align*}
& \sup_{(v,\delta) \in \Omega} \Phi(v,\delta) \\
& \leq 2 \sqrt{\frac{10 d \log(Rk) + \log(2/\gamma)}{d}} \left(\left(2 R \sqrt{\frac{\log(d)}{d}}\right)^{1/4} + \sqrt{\frac{10 d \log(Rk) + \log(2/\gamma)}{k}} \right) \\
& \,\, + 20 \sqrt{\frac{d}{k}} \sqrt{\log 4k + \frac{\log 8/\gamma}{d}} \\
&\leq 3 \sqrt{\frac{10 d \log(Rk) + \log(2/\gamma)}{d}} \left(\left(2 R \sqrt{\frac{\log(d)}{d}}\right)^{1/4} + \sqrt{\frac{10 d \log(Rk) + \log(2/\gamma)}{k}} \right)\,,
\end{align*}
which concludes the proof up to straightforward algebraic manipulations.
\end{proof}

\begin{proof}[of Theorem \ref{thm:relu}]
The proof is the same as for Theorem \ref{thm:smooth} with Lemma \ref{lem:fixedrelu} instead of Lemma \ref{lem:fixedsmooth} and Lemma \ref{lem:hessrelu2} instead of Lemma \ref{lem:hesssmooth2}.
\end{proof}

\section{Experiments} \label{sec:exp}
\paragraph{Setting} In order to verify our theoretical findings, we ran some experiments to measure empirically the values of $\|\nabla f(x)\|$ and the probability of finding an adversarial example in that direction. More precisely, we take a random point $x$ of norm $\sqrt{d}$ and initialize a network using the procedure described in Section~\ref{sec:intro}. We then find the smallest $\eta$ such that a gradient step $\eta \nabla f(x)$ changes the sign of the function. $\eta$ is of the opposite sign of $f(x)$ and we limit our search to $|\eta| < 20$. We explore various values of $d$ and $k$. We also consider deeper networks with $L = 1$ through $L = 6$ hidden layers. All the hidden layers are of width $k$.

\paragraph{Results} Figure~\ref{fig:length1} shows the average of the smallest $\eta$ required to switch the sign of the function. We note that the average only includes cases where an $\eta$ was indeed found. Figure~\ref{fig:grad1} shows the gradient norm in $x$ (all cases included). As we see, both the smallest $\eta$ and the gradient norm are approximately constant both in $d$ and in $k$. This finding also holds for deeper networks (see Appendix~\ref{app:more_results}).
In Figure~\ref{fig:switched}, we show the fraction of examples (out of $10,000$ samples) whose sign is switched for $|\eta| < 20$. We see that with $L = 1$ and values of $d$ and $k$ larger than 50, $100\%$ of samples are switched. This confirms our theoretical results. Additionally, we also observe that for deeper networks, the same statement holds. The values of $d$ and $k$ at which $100\%$ switching is reached appears to grow with $L$\footnote{Due to GPU memory limitations, $k$ could not reach 1,000,000 for deeper networks.}.

\begin{figure}
\begin{subfigure}{.5\textwidth}
    \centering
    \includegraphics[width=0.99\linewidth]{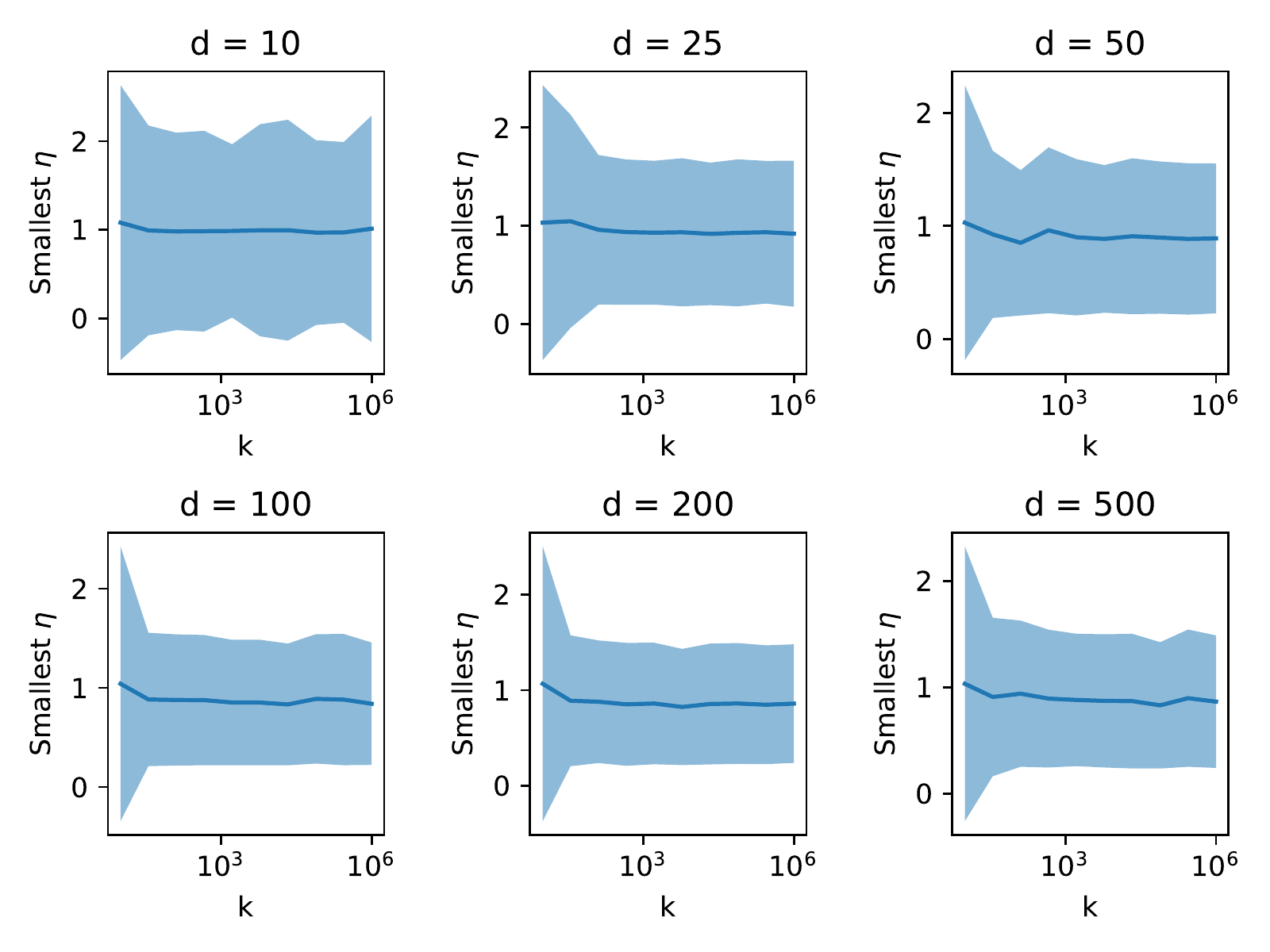}
    \caption{Smallest step size $\eta$}
    \label{fig:length1}
\end{subfigure}
\begin{subfigure}{.5\textwidth}
    \centering
    \includegraphics[width=0.99\linewidth]{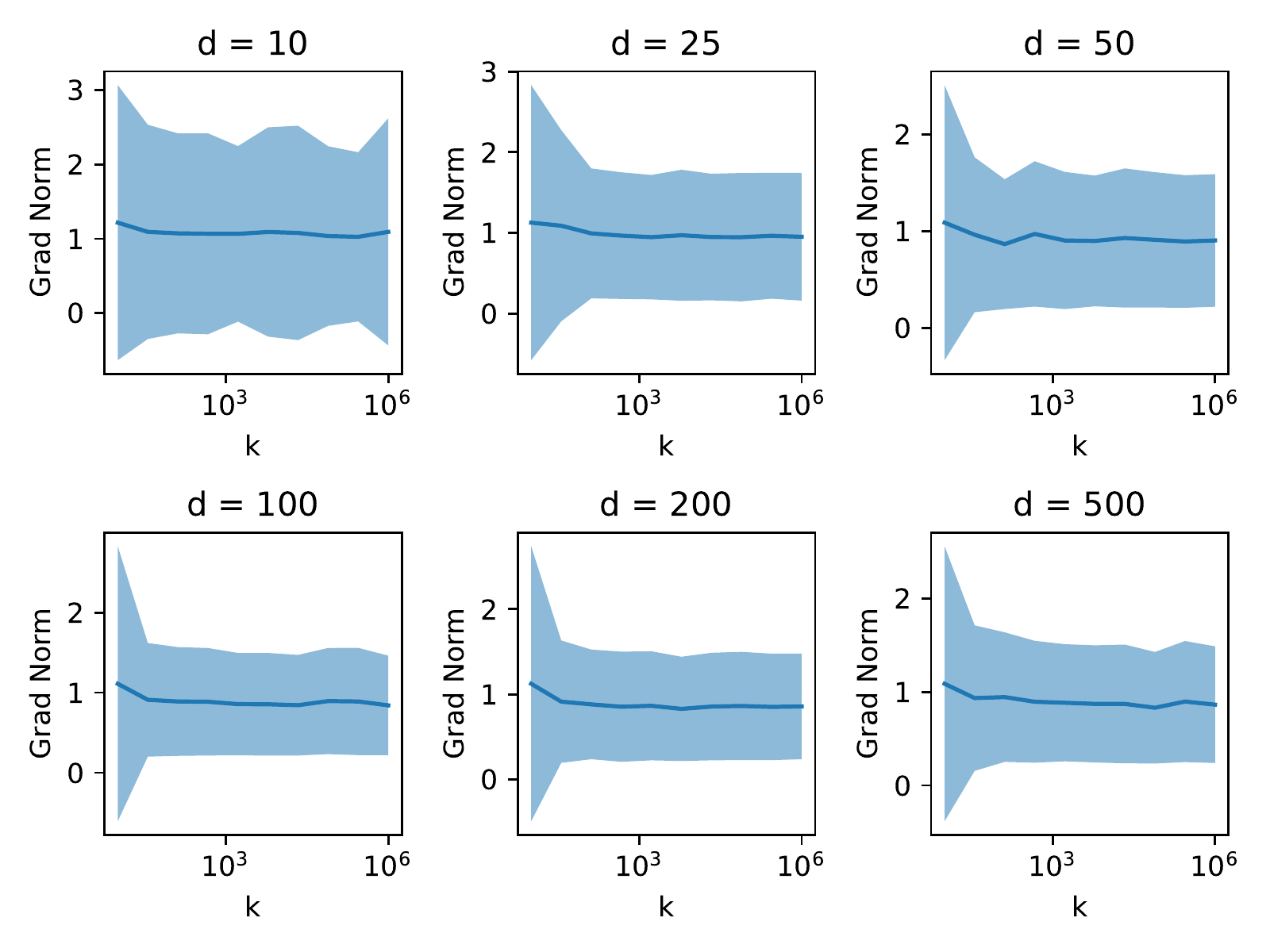}
    \caption{Norm of the gradient}
    \label{fig:grad1}
\end{subfigure}
\caption{Smallest step-size $\eta$ switching the prediction (\textbf{left}) and average gradient norm $\|\nabla f(x)\|$ (\textbf{right}) for L = 1. Averages over $100$ network initializations and $100$ values of $x$ per initialization. The colored area represents one standard deviation.}
\end{figure}

\begin{figure}
\begin{subfigure}{.32\textwidth}
    \centering
    \includegraphics[width=0.95\linewidth]{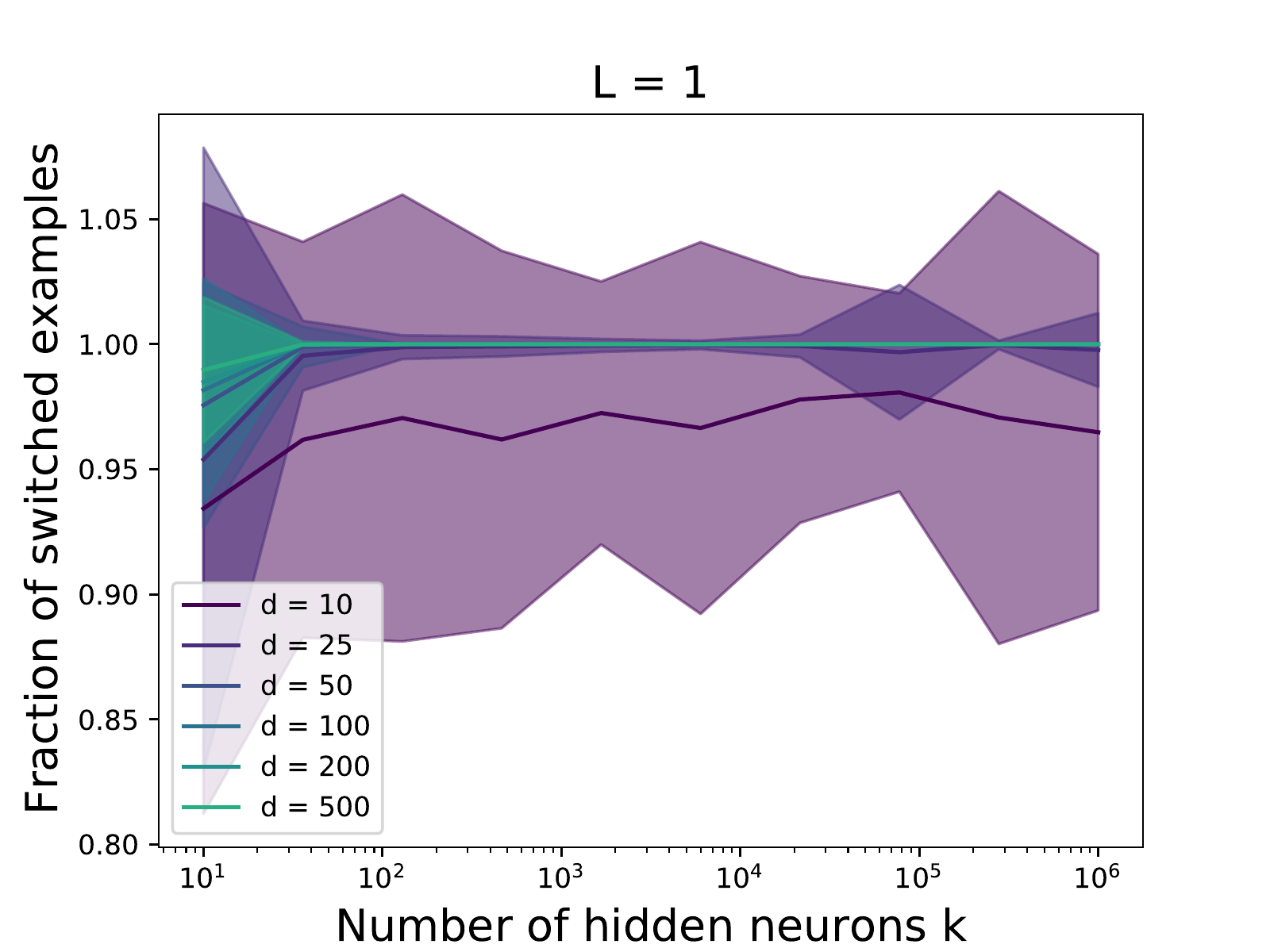}
    \label{fig:switched1}
\end{subfigure}
\begin{subfigure}{.32\textwidth}
    \centering
    \includegraphics[width=0.95\linewidth]{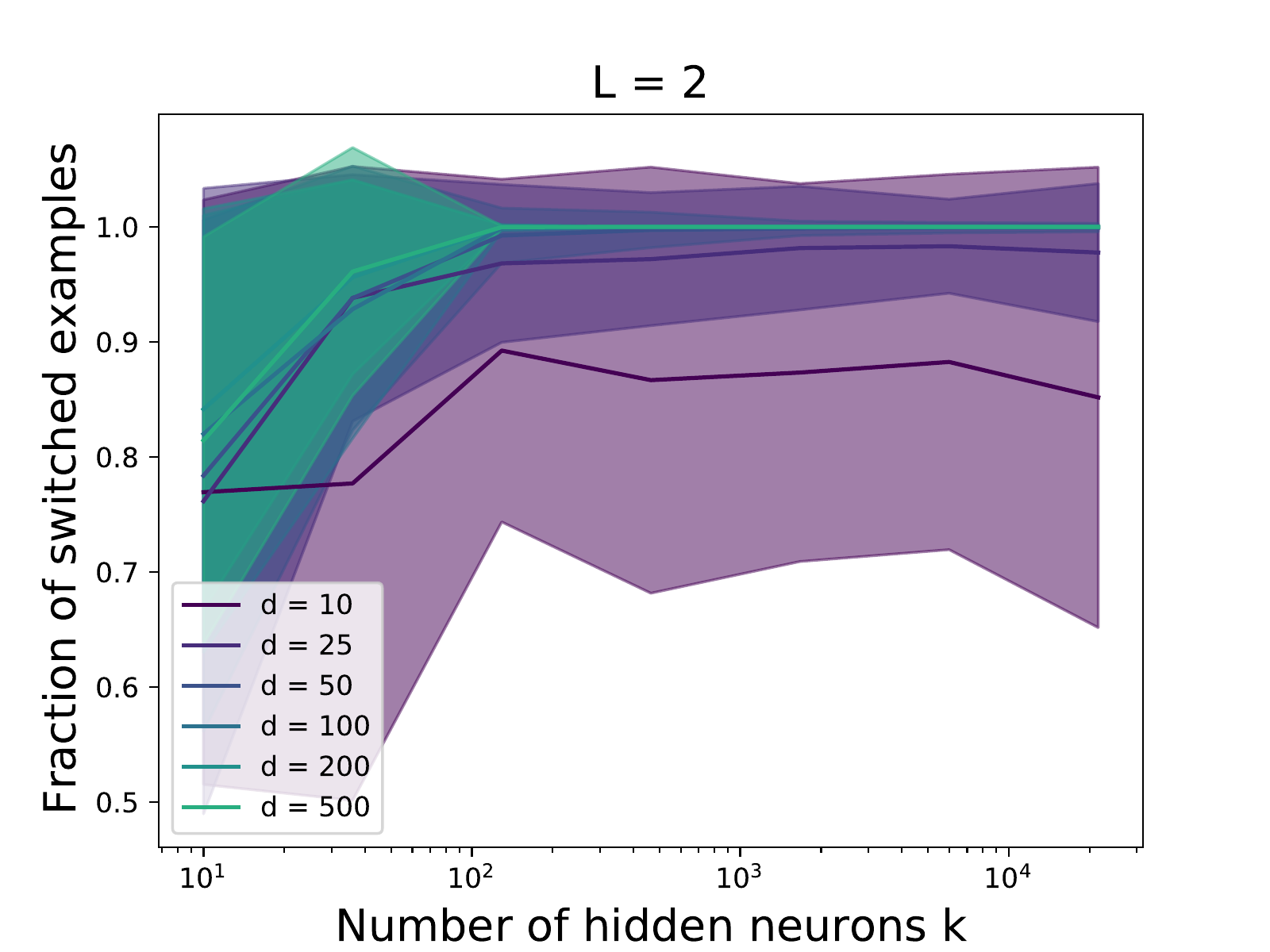}
    \label{fig:switched2}
\end{subfigure}
\begin{subfigure}{.32\textwidth}
    \centering
    \includegraphics[width=0.95\linewidth]{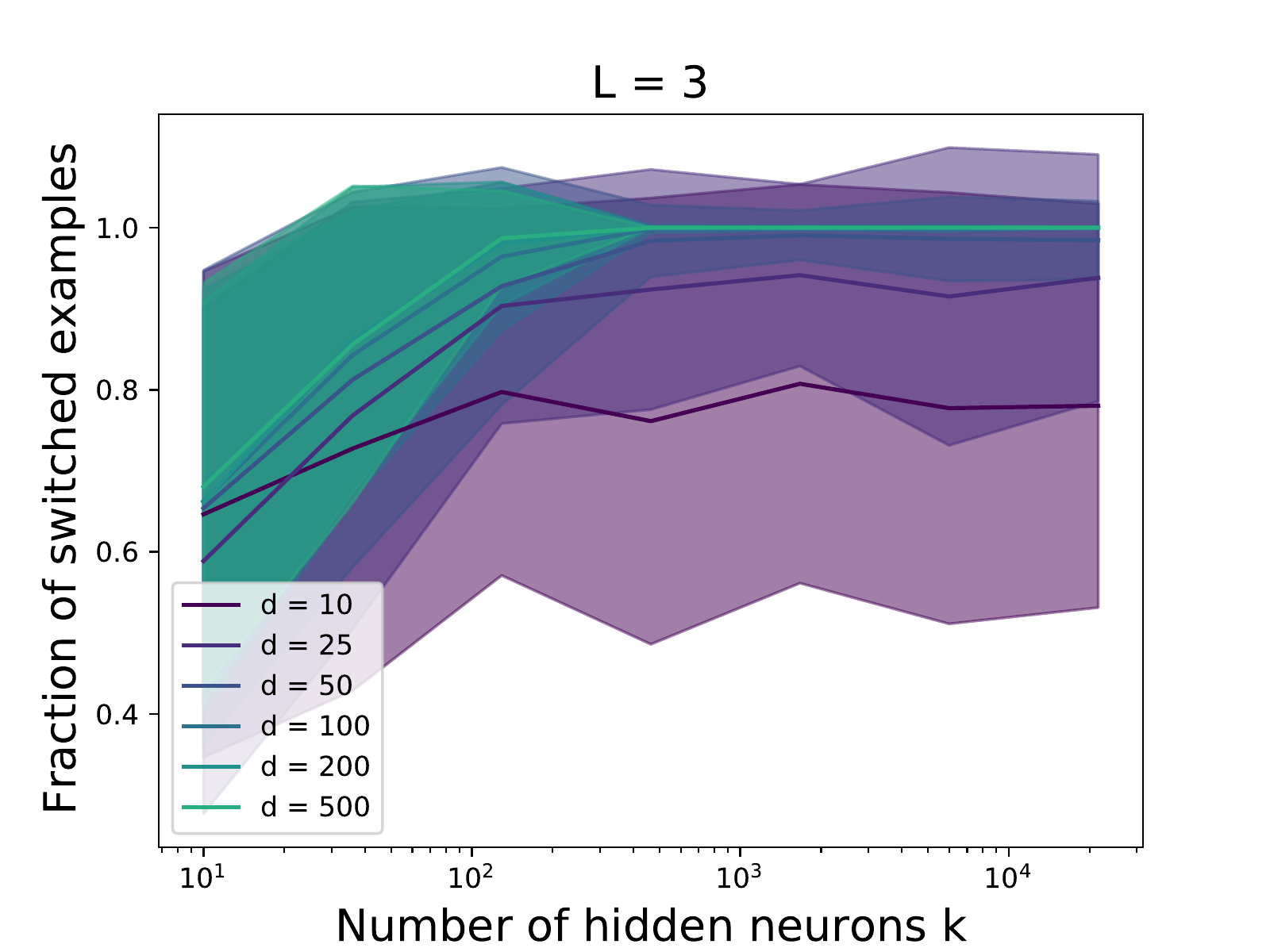}
    \label{fig:switched3}
\end{subfigure}
\begin{subfigure}{.32\textwidth}
    \centering
    \includegraphics[width=0.95\linewidth]{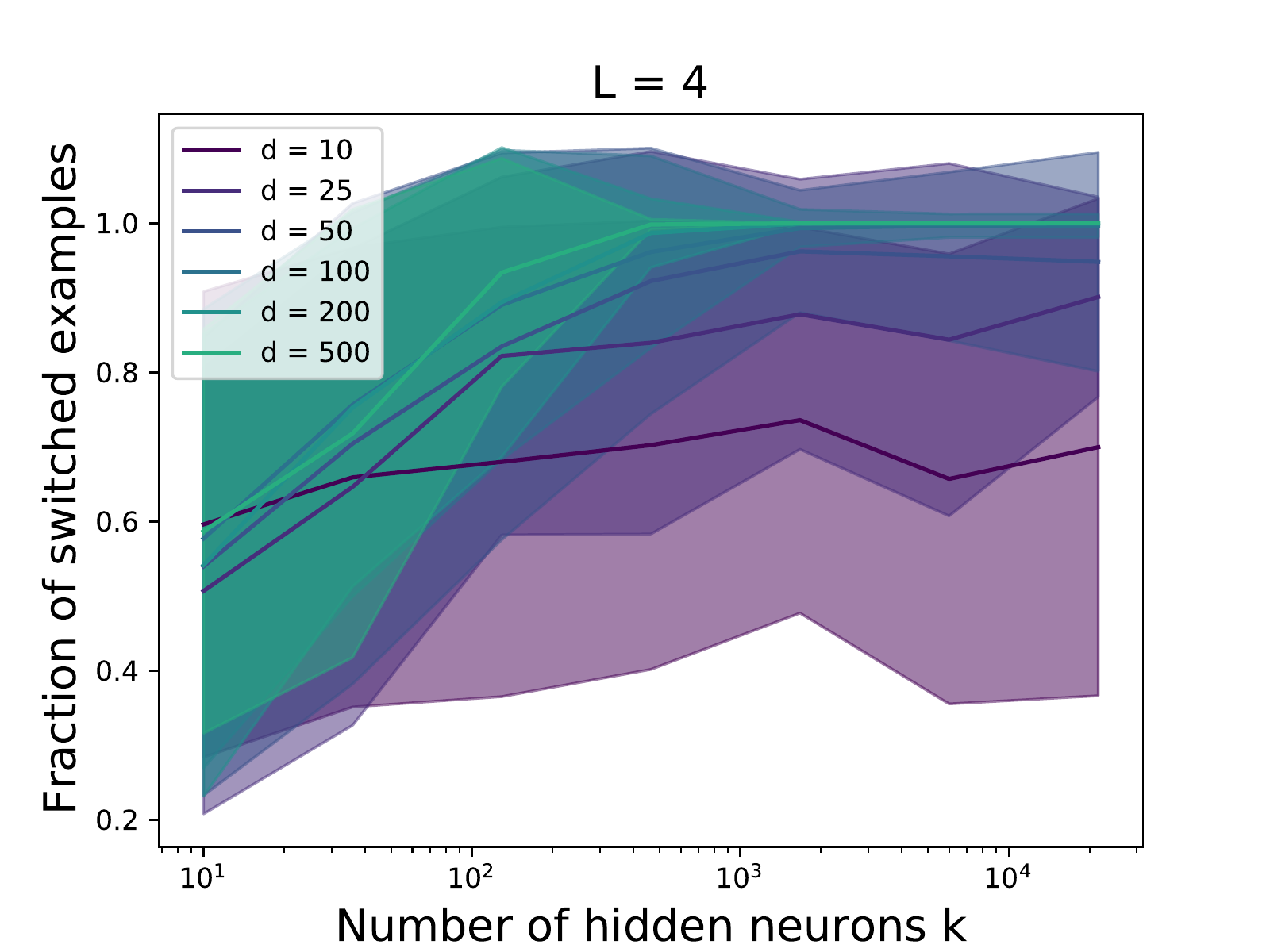}
    \label{fig:switched4}
\end{subfigure}
\begin{subfigure}{.32\textwidth}
    \centering
    \includegraphics[width=0.95\linewidth]{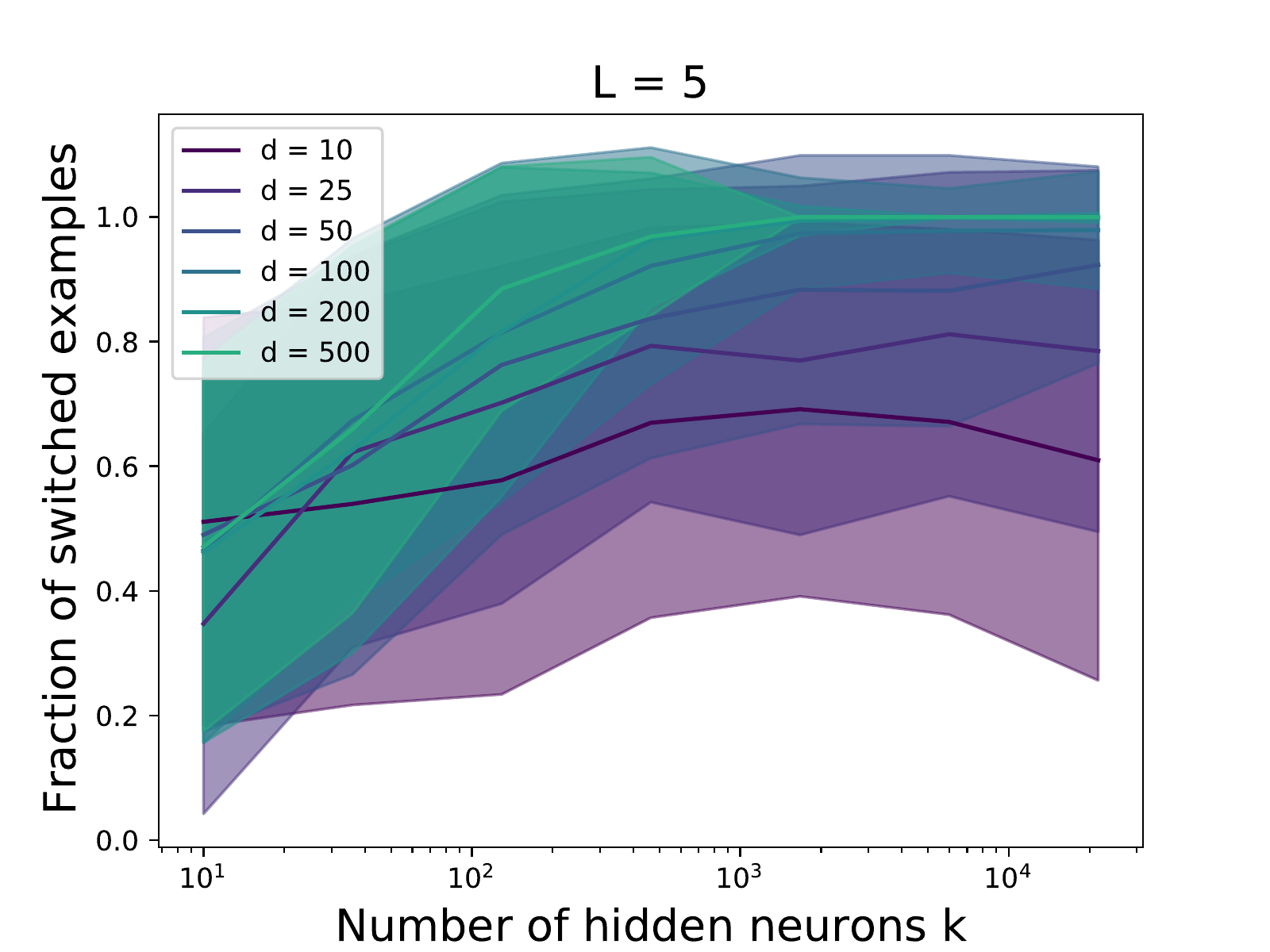}
    \label{fig:switched5}
\end{subfigure}
\begin{subfigure}{.32\textwidth}
    \centering
    \includegraphics[width=0.95\linewidth]{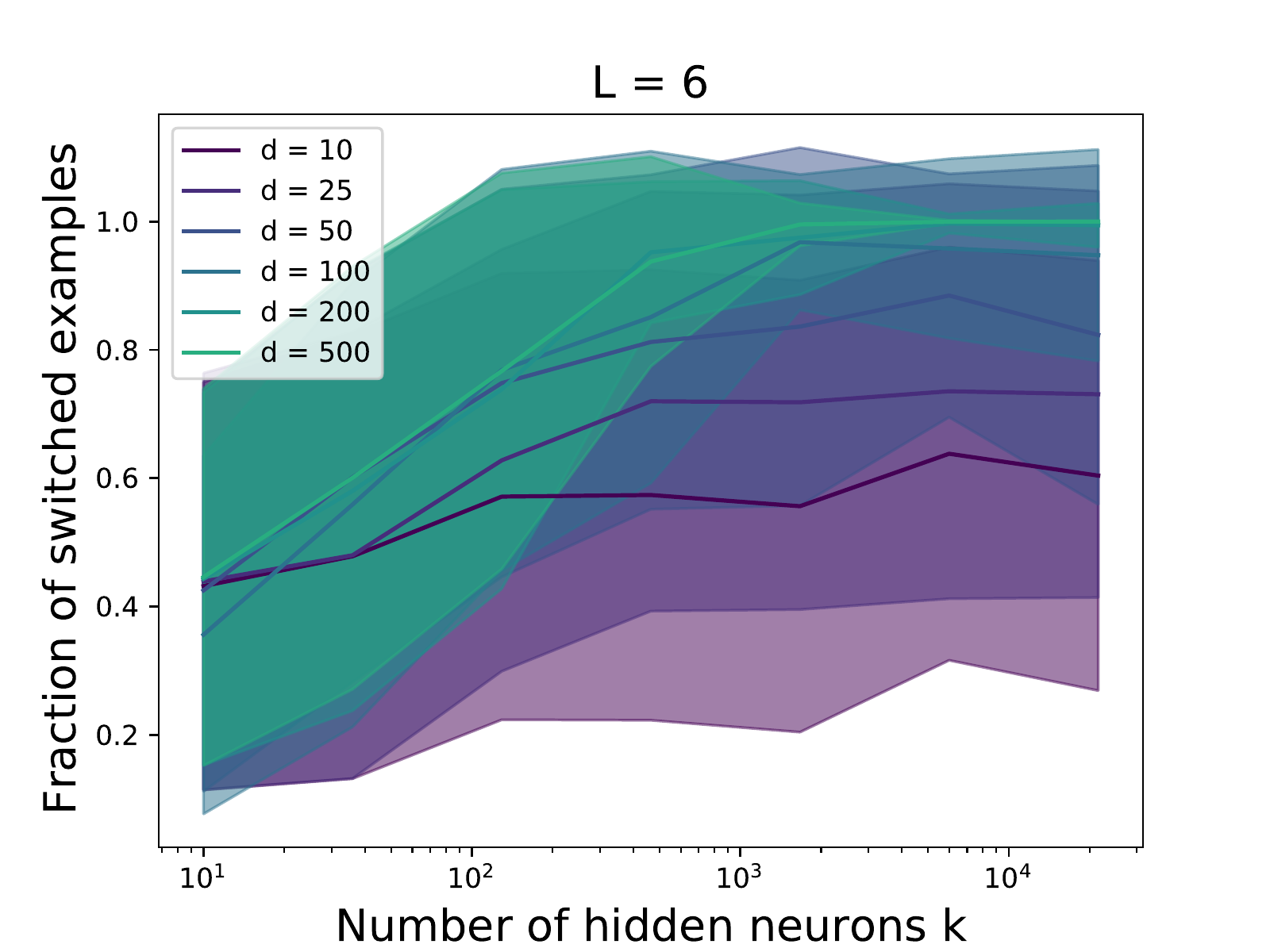}
    \label{fig:switched5}
\end{subfigure}
\caption{Fraction of inputs with an adversarial example found with $\eta < 20$.}
\label{fig:switched}
\end{figure}

\subsubsection*{Acknowledgment}
We thank Mark Sellke for pointing out to us the reference \cite{arous2020geometry}, and Peter Bartlett for several discussions on this problem.

\bibliographystyle{plainnat}
\bibliography{neuralbib}

\newpage
\appendix
\section{Appendix}
\label{app:more_results}

For the sake of completeness, we report in Fig.~\ref{fig:app:L=2}-\ref{fig:app:L=6} the smallest $\eta$ to switch the sign of the prediction and the gradient norm at $x$ for depths $L \in \{2,\ldots,6\}$. In all our plots, the results are Averaged over $100$ network initializations and $100$ values of $x$ per initialization and the colored area represents one standard deviation 

\begin{figure}[h]
\begin{subfigure}{.5\textwidth}
    \centering
    \includegraphics[width=0.9\linewidth]{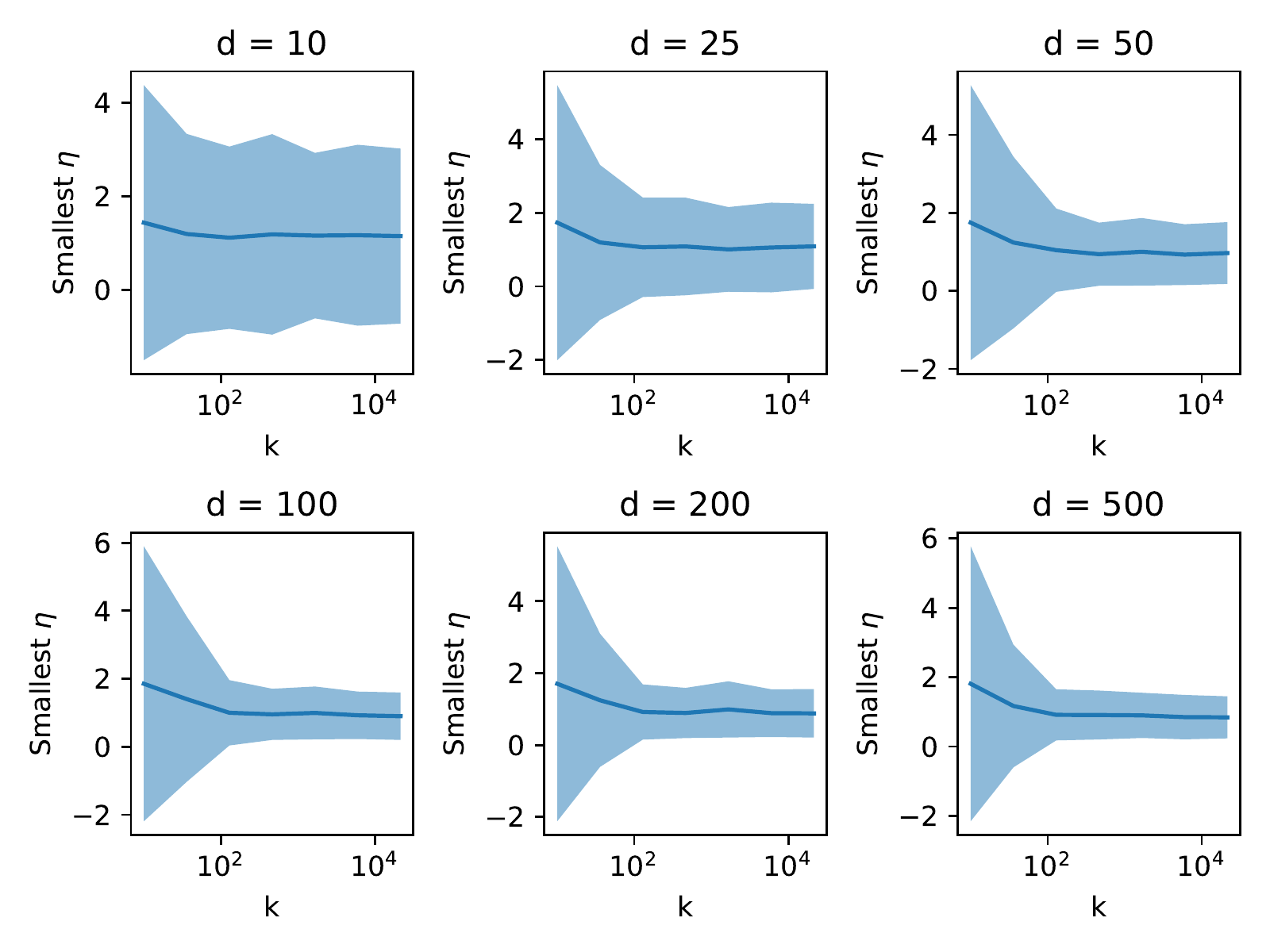}
    \caption{Smallest step size $\eta$}
    \label{fig:length2}
\end{subfigure}
\begin{subfigure}{.5\textwidth}
    \centering
    \includegraphics[width=0.9\linewidth]{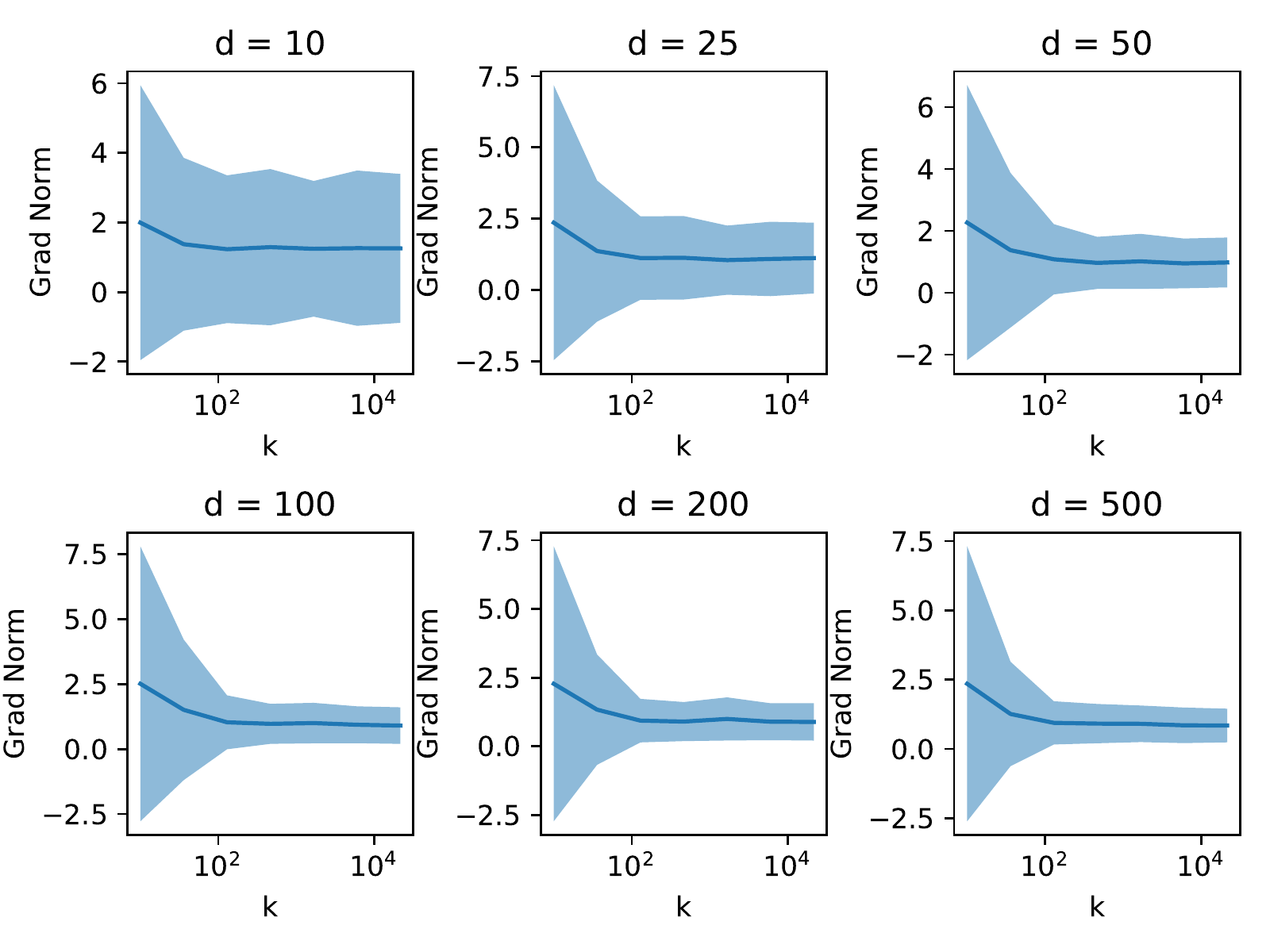}
    \caption{Norm of the gradient}
    \label{fig:grad2}
\end{subfigure}
\caption{Smallest $\eta$ switching the prediction and average gradient norm $\|\nabla f(x)\|$ for L = 2.}
\label{fig:app:L=2}
\end{figure}

\begin{figure}[h]
\begin{subfigure}{.5\textwidth}
    \centering
    \includegraphics[width=0.9\linewidth]{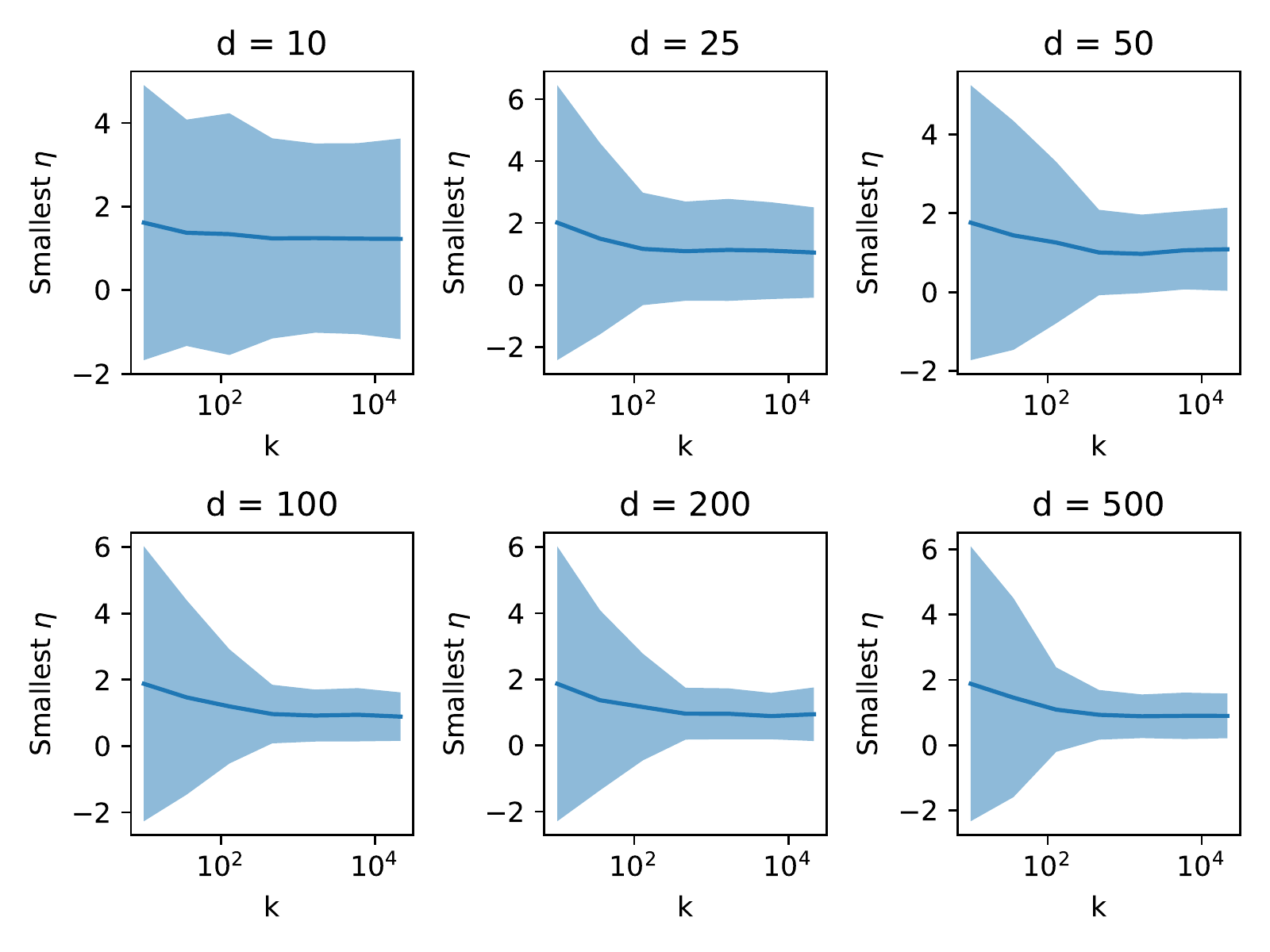}
    \caption{Smallest step size $\eta$}
    \label{fig:length3}
\end{subfigure}
\begin{subfigure}{.5\textwidth}
    \centering
    \includegraphics[width=0.9\linewidth]{figures/grad_1.pdf}
    \caption{Norm of the gradient}
    \label{fig:grad3}
\end{subfigure}
\caption{Smallest $\eta$ switching the prediction and average gradient norm $\|\nabla f(x)\|$ for L = 3.}
\end{figure}
\vspace{-1cm}
\begin{figure}[h]
\begin{subfigure}{.5\textwidth}
    \centering
    \includegraphics[width=0.9\linewidth]{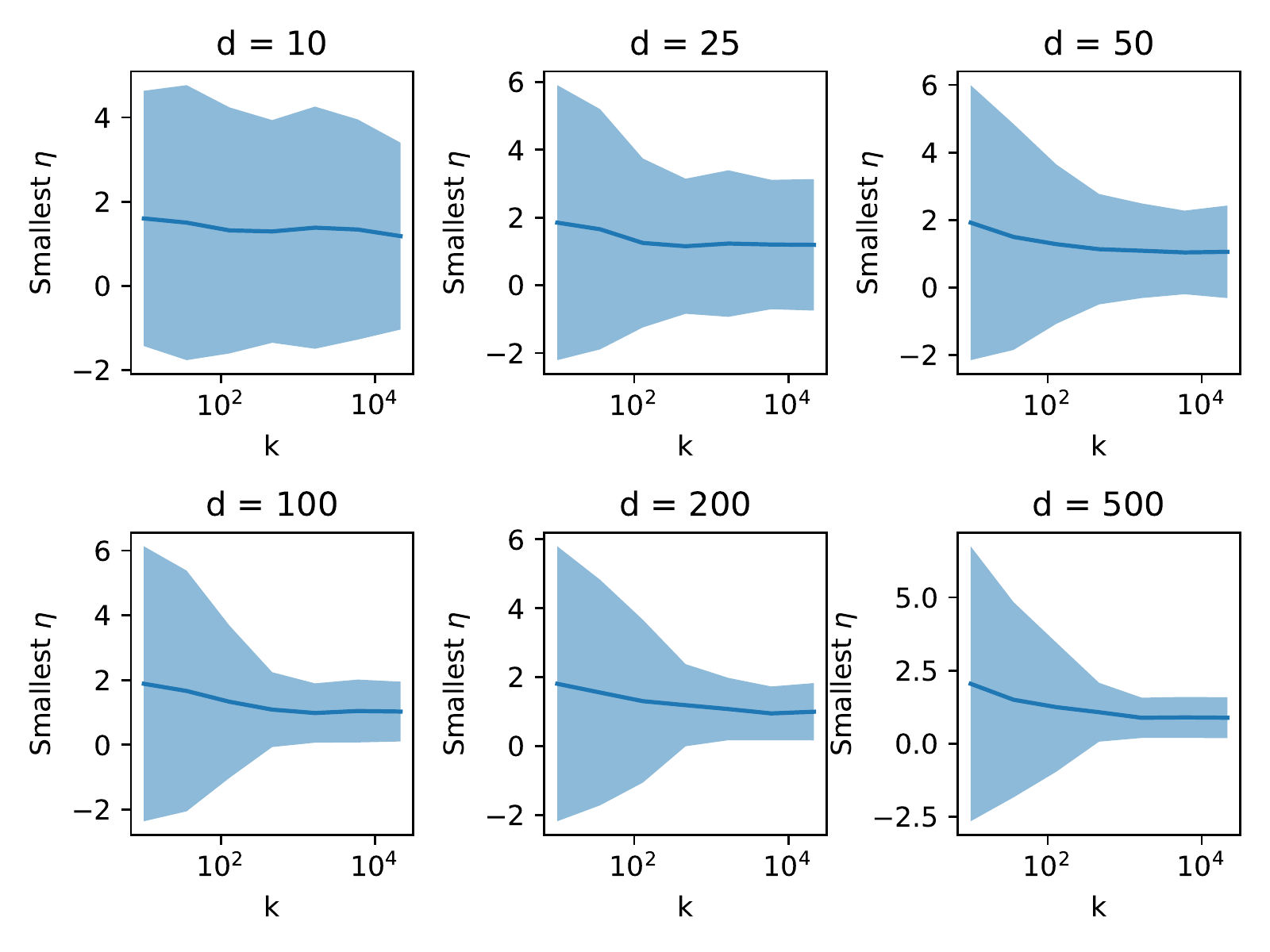}
    \caption{Smallest step size $\eta$}
    \label{fig:length4}
\end{subfigure}
\begin{subfigure}{.5\textwidth}
    \centering
    \includegraphics[width=0.9\linewidth]{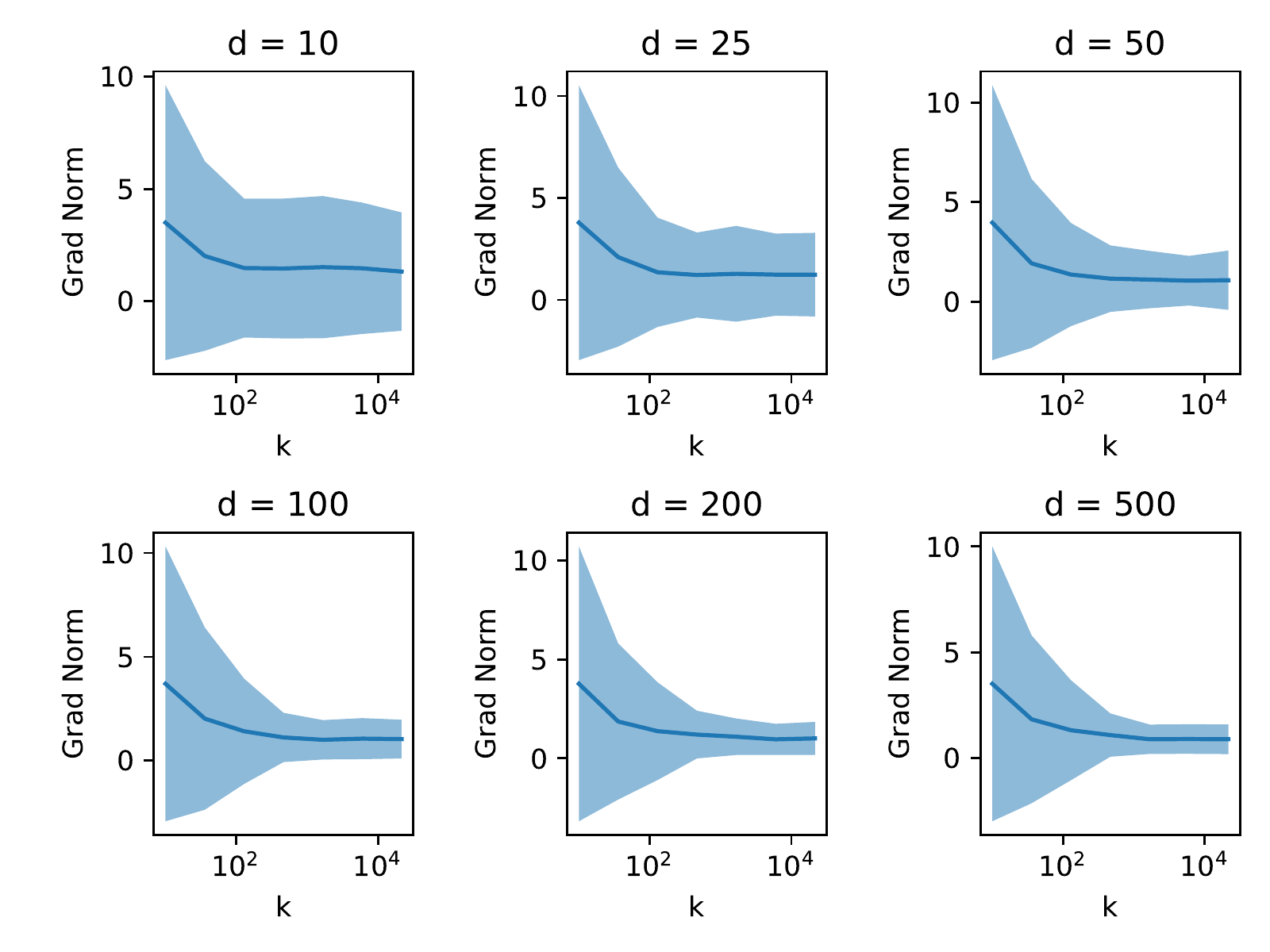}
    \caption{Norm of the gradient}
    \label{fig:grad4}
\end{subfigure}
\caption{Smallest $\eta$ switching the prediction and average gradient norm $\|\nabla f(x)\|$ for L=4.}
\end{figure}
\vspace{-1cm}
\begin{figure}[h]
\begin{subfigure}{.5\textwidth}
    \centering
    \includegraphics[width=0.9\linewidth]{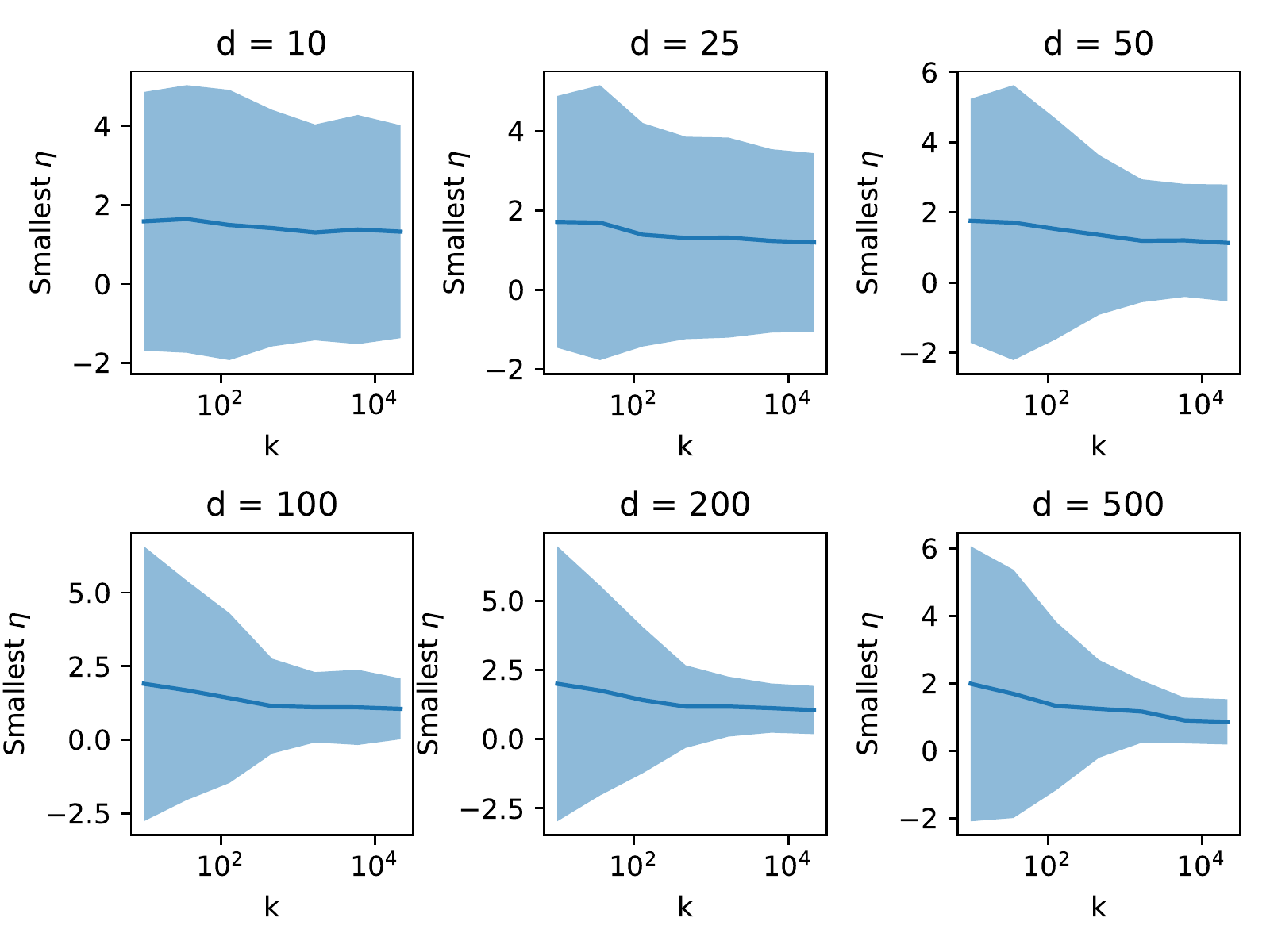}
    \caption{Smallest step size $\eta$}
    \label{fig:length5}
\end{subfigure}
\begin{subfigure}{.5\textwidth}
    \centering
    \includegraphics[width=0.9\linewidth]{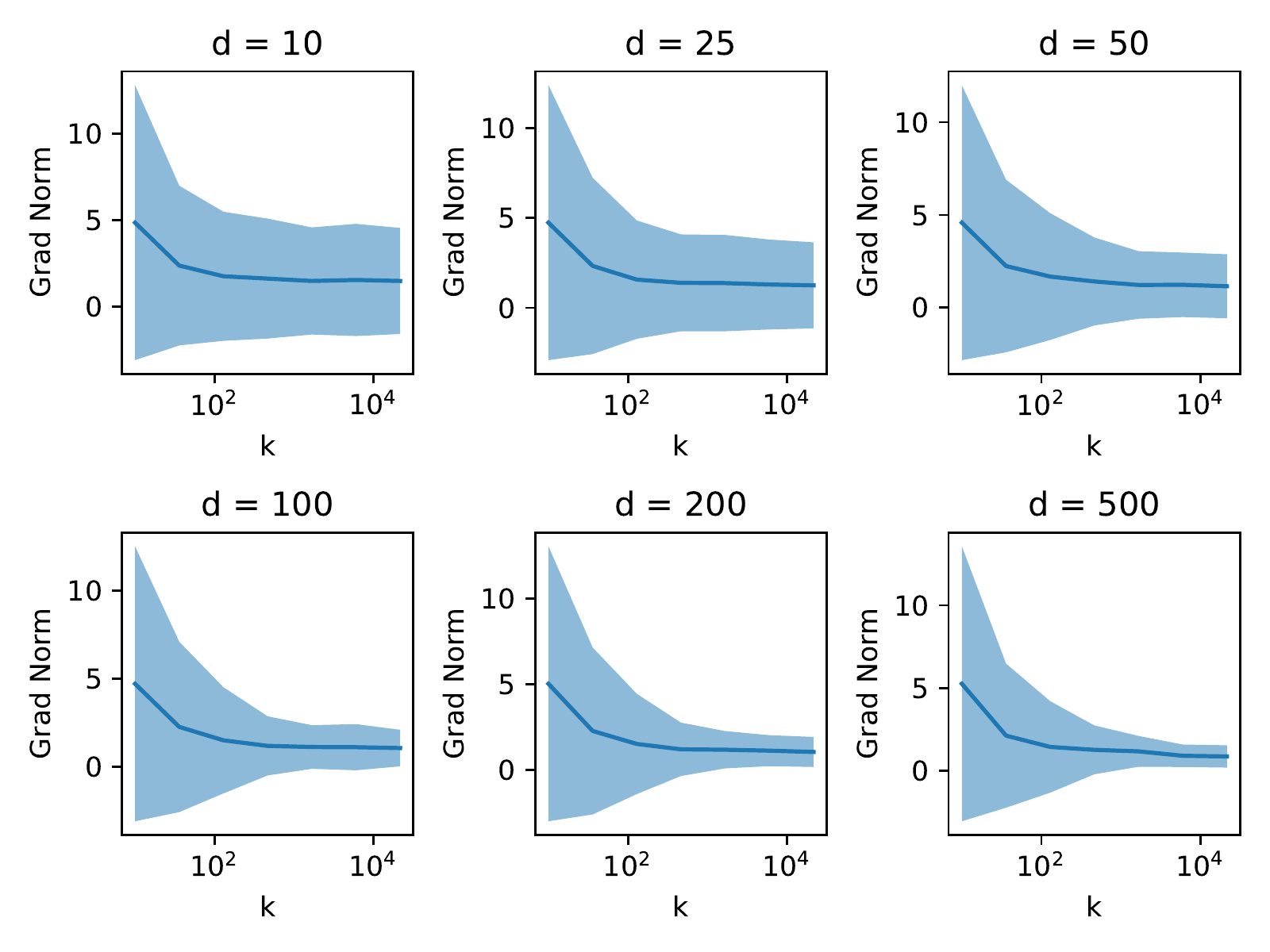}
    \caption{Norm of the gradient}
    \label{fig:grad5}
\end{subfigure}
\caption{Smallest $\eta$ switching the prediction and average gradient norm $\|\nabla f(x)\|$ for L = 5. }
\end{figure}
\vspace{-1cm}
\begin{figure}[h]
\begin{subfigure}{.5\textwidth}
    \centering
    \includegraphics[width=0.9\linewidth]{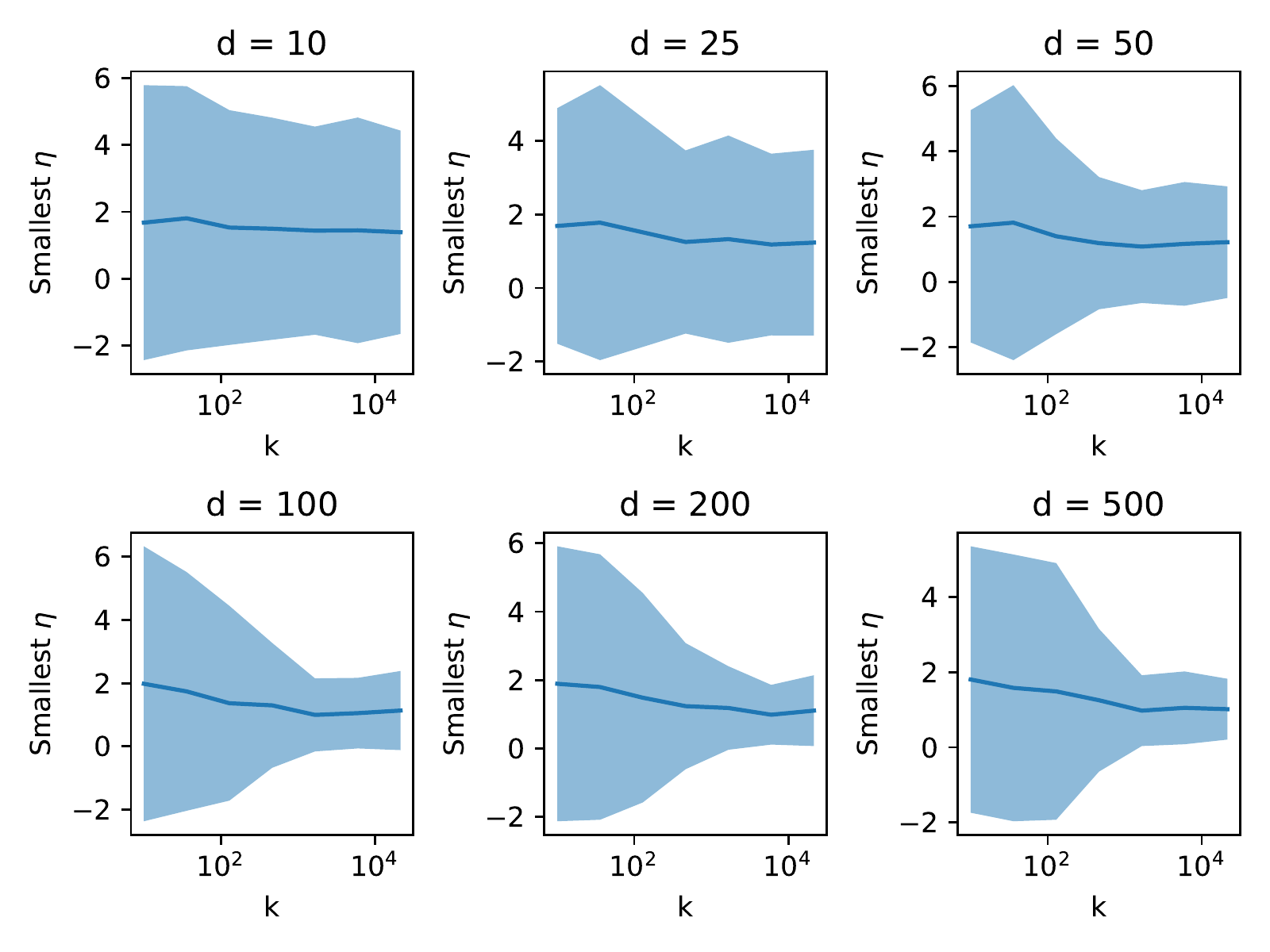}
    \caption{Smallest step size $\eta$}
    \label{fig:length6}
\end{subfigure}
\begin{subfigure}{.5\textwidth}
    \centering
    \includegraphics[width=0.9\linewidth]{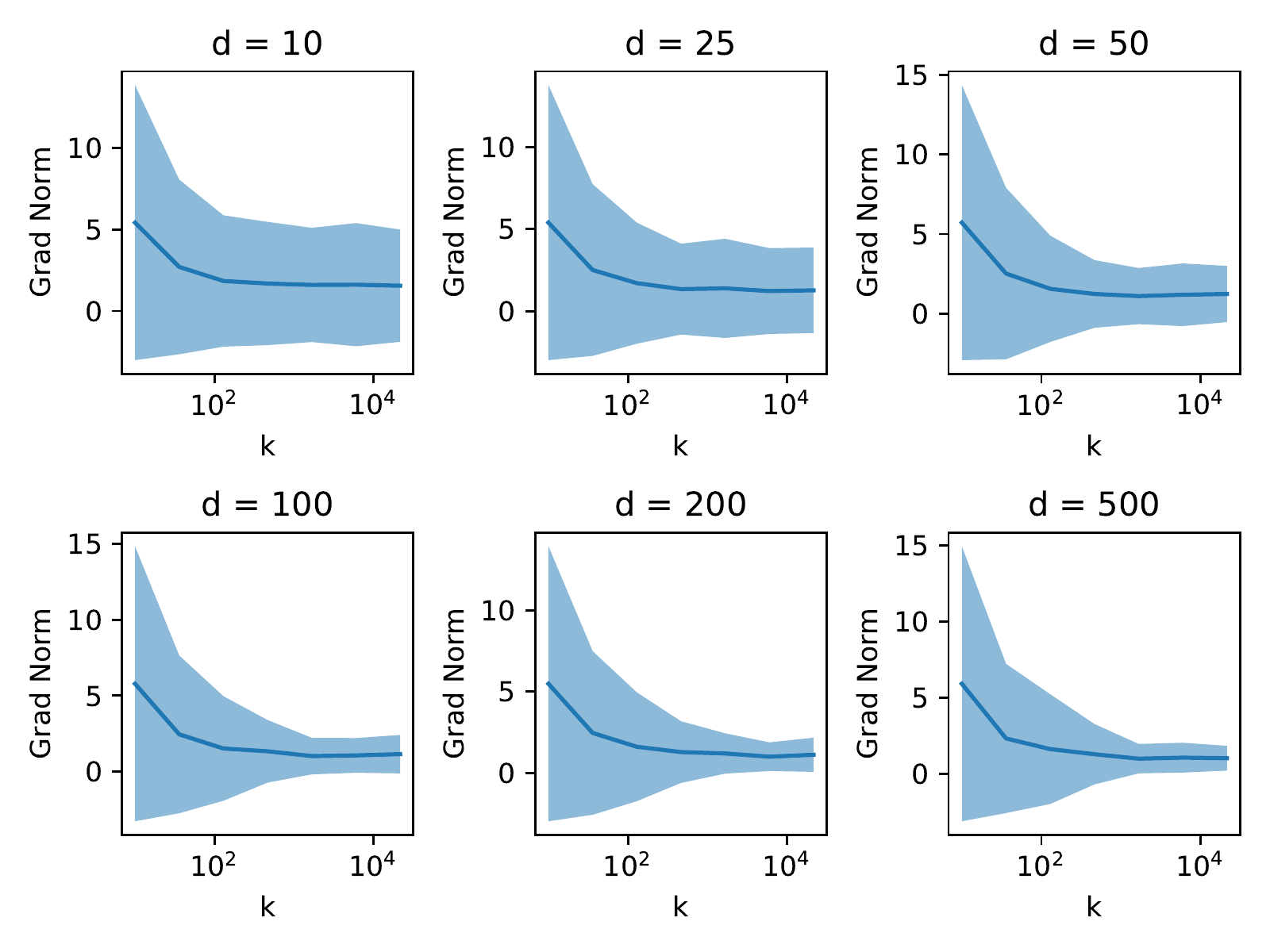}
    \caption{Norm of the gradient}
    \label{fig:grad}
\end{subfigure}
\caption{Smallest $\eta$ switching the prediction and average gradient norm $\|\nabla f(x)\|$ for L =6.}
\label{fig:app:L=6}
\end{figure}

\end{document}